\def\year{2018}\relax
\documentclass[letterpaper]{article} %DO NOT CHANGE THIS
\usepackage{aaai18} 
\usepackage{times}  %Required
\usepackage{helvet}  %Required
\usepackage{courier}  %Required
\usepackage{url}  %Required
\usepackage{graphicx}  %Required
\usepackage{amsmath,bm}
\usepackage{multirow,subfigure}
\usepackage{color}
\newcommand\jonathan[1]{\textcolor{green}{[JK: #1]}}
\newcommand\se[1]{\textcolor{red}{[SE: #1]}}
\newcommand\ashish[1]{\textcolor{blue}{[AS: #1]}}
 \renewcommand\jonathan[1]{} %uncomment to hide jonathan
 \renewcommand\se[1]{} %uncomment to hide se
 \renewcommand\ashish[1]{} %uncomment to hide ashish

\usepackage[english]{babel}
\usepackage[utf8x]{inputenc}
\usepackage[T1]{fontenc}
\usepackage{amsmath}
\usepackage{mathtools}

\usepackage{url}            % simple URL typesetting
\usepackage{booktabs}       % professional-quality tables
\usepackage{amsfonts}       % blackboard math symbols
\usepackage{nicefrac}       % compact symbols for 1/2, etc.
\usepackage{microtype}      % microtypography

\usepackage{graphicx}
\usepackage{placeins}
\usepackage{algorithm}
\usepackage{algorithmic}

\usepackage{amssymb}
\usepackage{amsmath}
\usepackage{amsthm}

\newcommand{\beq}[1][\vspace{0.3em}]{#1\begin{equation}}
\newcommand{\eeq}{\end{equation}}

\newcommand{\bit}{\vspace{0mm}\begin{itemize}}
\newcommand{\eit}{\vspace{0mm}\end{itemize}}
\newcommand{\ben}{\vspace{0mm}\begin{enumerate}}
\newcommand{\een}{\vspace{0mm}\end{enumerate}}
\newtheorem{theorem}{Theorem}
\newtheorem{prop}{Proposition}

\newtheorem{lemma}{Lemma} 
\newtheorem{corollary}{Corollary}[lemma]
\newtheorem{definition}{Definition}
\newtheorem{assumption}{Assumption}

\newcommand{\E}{\mathbb{E}}

\newcommand{\R}{\mathcal{R}} %Weighted Rademacher symbol
\newcommand\mydots{\makebox[1em][c]{.\hfil.\hfil.}}
\DeclareMathOperator*{\argmax}{arg\,max}
\DeclareMathOperator*{\argmin}{arg\,min}

\usepackage{amsmath}
\usepackage{graphicx}
\graphicspath{ {experiment_plots/} }
\usepackage[colorinlistoftodos]{todonotes}
\usepackage{changepage}

\newcommand{\namecite}[1]{\citeauthor{#1}~\shortcite{#1}}
\newcommand{\citenobrackets}[1]{\citeauthor{#1},~\citeyear{#1}}
\newcommand\citep{\cite}
\newcommand\citet{\namecite}
\newcommand\citealp{\citenobrackets}

\title{
Approximate Inference via Weighted Rademacher Complexity
}

\author{
  Jonathan Kuck\\
  Computer Science Department\\
  Stanford University\\
  \texttt{jkuck@cs.stanford.edu}
  \\\And
  Ashish Sabharwal\\
  Allen Institute for Artificial Intelligence\\
  Seattle, WA\\
  \texttt{ashishs@allenai.org}
  \\\And
  Stefano Ermon\\
  Computer Science Department\\
  Stanford University\\
  \texttt{ermon@cs.stanford.edu}
}

\begin{document}
\maketitle

%-------------------------------------------------------------------------------------------------------------
\begin{abstract}
Rademacher complexity is often used to characterize the learnability of a hypothesis class and is known to be related to the class size. 
We leverage this observation and introduce a new technique for estimating the size of an arbitrary weighted set, defined as the sum of weights of all elements in the set. 
Our technique provides upper and lower bounds on a novel generalization of Rademacher complexity to the weighted setting in terms of the weighted set size. 
  This generalizes Massart's Lemma, a known upper bound on the Rademacher complexity 
 in terms of the unweighted set size.
  We show that the weighted Rademacher complexity can be estimated by solving a randomly perturbed optimization problem, allowing us to derive high-probability bounds on the size of any weighted set.
We apply our method to the problems of calculating the partition function of an Ising model and computing propositional model counts (\#SAT).  Our experiments demonstrate that we can produce tighter bounds than competing methods in both the weighted and unweighted settings.
\end{abstract}

%-------------------------------------------------------------------------------------------------------------
\section{Introduction}
A wide variety of problems can be reduced to computing the sum of (many) non-negative numbers.  These include calculating the partition function of a graphical model, propositional model counting (\#SAT), and calculating the permanent of a non-negative matrix. Equivalently, each can be viewed as computing the discrete integral of a non-negative weight function. Exact summation, however, is generally intractable due to the curse of dimensionality \cite{bellman2015adaptive}.

As alternatives to exact computation, variational methods \cite{jordan1998introduction,wainwright2008graphical} and sampling \cite{jerrum1996markov,madras2002lectures} are popular approaches for approximate summation. However, they generally do not guarantee the estimate's quality.  

An emerging line of work estimates and formally bounds propositional model counts or, more generally, discrete integrals \cite{ermon2013taming,chakraborty2013scalable,ermon2014low,zhao2016closing}.  
These approaches reduce the problem of integration to solving a small number of optimization problems involving the same weight function but subject to additional random constraints introduced by a random hash function.    
This results in approximating the \#P-hard problem of exact summation \cite{valiant1979complexity} using the solutions of NP-hard optimization problems.

Optimization can be performed efficiently for certain classes of weight functions, such as those involved in the computation of the permanent of a non-negative matrix.  If instead of summing (permanent computation) we maximize the same weight function, we obtain a maximum weight matching problem, which is in fact solvable in polynomial time~\cite{kuhn1955hungarian}.
However, adding hash-based constraints makes the maximum matching optimization problem intractable, which limits the application of randomized hashing approaches~\cite{uai13LPCount}.  On the other hand, there do exist fully polynomial-time randomized approximation schemes (FPRAS) for non-negative permanent computation~\cite{jerrum2004polynomial,bezakova2006accelerating}.  This gives hope that approximation schemes may exist for other counting problems even when optimization with hash-based constraints is intractable.

We present a new method for approximating and bounding the size of a general weighted set (i.e., the sum of the weights of its elements) using geometric arguments based on the set's shape.  Our approach, rather than relying on hash-based techniques, establishes a novel connection with Rademacher complexity~\cite{shalev2014understanding}.  This generalizes geometric approaches developed for the unweighted case to the weighted setting, such as the work of \citet{barvinok1997approximate} who uses similar reasoning but without connecting it with Rademacher complexity. 
In particular, we first generalize Rademacher complexity to weighted sets. 
While Rademacher complexity is defined as the maximum of the sum of Rademacher variables over a set, \emph{weighted} Rademacher complexity also accounts for the weight of each element in the set. Just like Rademacher complexity is related to the size of the set, we show that \emph{weighted} Rademacher complexity is related to the total weight of the set.
Further, it can be estimated by solving multiple instances of a maximum weight optimization problem, subject to random Rademacher perturbations.
Notably, the resulting optimization problem turns out to be computationally much simpler than that required by the aforementioned randomized hashing schemes.  In particular, if the weight function is log-supermodular, the corresponding weighted Rademacher complexity can be estimated efficiently, as our perturbation does not change the original optimization problem's complexity~\cite{orlin2009faster,bach2013learning}.

Our approach most closely resembles a recent line of work involving the Gumbel distribution \cite{hazan2012partition,hazan2013sampling,hazan2016high,pmlr-v70-balog17a,mussmann2016learning,mussmann2017fast}.  There, the Gumbel-max idea is used to bound the partition function by performing MAP inference on a model where the unnormalized probability of each state is perturbed by random noise variables sampled from a Gumbel distribution.  While very powerful, exact application of the Gumbel method is impractical, as it requires exponentially many independent random perturbations.  One instead uses local approximations of the technique.

Empirically, on spin glass models we show that our technique yields tighter upper bounds and similar lower bounds compared with the Gumbel method, given similar computational resources. On a suite of \#SAT model counting instances our approach generally produces comparable or tighter upper and lower bounds given limited computation.%On a suite of \#SAT model counting instances our approach generally results in substantially tighter upper bounds given limited computation and, on certain classes of instances, significantly tighter lower bounds as well.

%-------------------------------------------------------------------------------------------------------------
\section{Background}
Rademacher complexity is an important tool used in learning theory to bound the generalization error of a hypothesis class~\citep{shalev2014understanding}.

\begin{definition}
\label{rademacher}
The Rademacher complexity of a set $A \subseteq \mathbb{R}^n$ is defined as:
\begin{equation}
R(A) \coloneqq \frac{1}{n} \E_{\mathbf{c}} \left[\sup_{\mathbf{a} \in A} \sum_{i=1}^n c_i a_i \right],
\label{eq:rademacher}
\end{equation}
where $\E_{\mathbf{c}}$ denotes expectation over $\mathbf{c}$, and $\mathbf{c}$ is sampled uniformly from $\{-1, 1\}^n$.
\end{definition}
As the name suggests, it is a measure of the complexity of set $A$ (which, in learning theory, is usually a hypothesis class). It measures how ``expressive'' $A$ is by evaluating how well we can ``fit'' to a random noise vector $\mathbf{c}$ by choosing the closest vector 
(or hypothesis) from $A$.
Intuitively, Rademacher complexity is related to $|A|$, the number of vectors in $A$, another crude notion of complexity of $A$. However, it also depends on how vectors in $A$ are arranged in the ambient space $\mathbb{R}^n$. A central focus of this paper will be establishing quantitative relationships between $R(A)$ and $|A|$.

A key property of Rademacher complexity that makes it extremely useful in learning theory is that it can be estimated using a small number of random noise samples $\mathbf{c}$ under mild conditions~\citep{shalev2014understanding}. The result follows from McDiarmid's inequality:  
\begin{prop}[\citealp{mcdiarmid1989method}]
\label{mcdiarmid}
Let $X_1, \mydots, X_m \in \mathcal{X}$ be independent random variables.
Let $f : \mathcal{X}^m \mapsto \mathbb{R}$ be a function that satisfies the bounded differences condition that $\forall i \in \{1,\mydots, m\}$ and $ \forall x_1,\mydots, x_m, x_i' \in \mathcal{X}$:
\begin{align*}
|f(x_1, \mydots, x_i, \mydots, x_m)-f(x_1, \mydots, x_i', \mydots, x_m)| \leq d_i. \\
\end{align*}
Then for all $\epsilon>0$
\scriptsize
\[
\Pr \Big[\big|f\left(X_1,\mydots, X_m\right)-\E\big[f\left(X_1,\mydots, X_m \right) \big] \big| \geq \epsilon \Big] \leq \exp \left( \frac{-2 \epsilon^2}{\sum_j d_j^2} \right).
\]
\normalsize
\end{prop}
McDiarmid's inequality says we can bound, with high probability, how far a function $f$ of random variables may deviate from its expected value, given that the function does not change much when the value of a single random variable is changed.  Because the function in  Eq.~(\ref{eq:rademacher}) satisfies this property~\citep{shalev2014understanding}, we can use Eq.~(\ref{eq:rademacher}) to bound $R(A)$ with high probability by computing the supremum for only a small number of noise samples $\mathbf{c}$. 

%-------------------------------------------------------------------------------------------------------------
\section{Problem Setup} \label{problem_setup}

In this section we formally define our problem and introduce the optimization oracle central to our solution.  Let $w:\{-1,1\}^n \rightarrow [0, \infty)$ be a non-negative weight function.  We consider the problem of computing the sum
\[
Z(w) = \sum_{\mathbf{x} \in \{-1,1\}^n} w(\mathbf{x}).
\]
Many problems, including computing the partition function of an undirected graphical model, where $w(\mathbf{x})$ is the unnormalized probability of state $\mathbf{x}$ (see \citet{koller2009probabilistic}), propositional model counting (\#SAT), and computing the permanent of a non-negative matrix can be reduced to calculating this sum.  The problem is challenging because explicit calculation requires summing over $2^n$ states, which is computationally intractable in cases of interest.   

Due to the general intractability of exactly calculating $Z(w)$, we focus on an efficient approach for estimating $Z(w)$ which additionally provides upper and lower bounds that hold with high probability.  Our method depends on the following assumption:
\begin{assumption} \label{opt_oracle}  We assume existence of an optimization oracle that can output the value
\begin{equation}
\delta(\mathbf{c}, w) = \max_{\mathbf{x} \in \{-1,1\}^n} \{\langle \mathbf{c},\mathbf{x} \rangle + \log w(\mathbf{x})\}
\label{eq:randomOpt}
\end{equation}
for any vector $\mathbf{c} \in \{-1,1\}^n$ and weight function $w:\{-1,1\}^n \rightarrow [0, \infty)$.
\end{assumption}
Note that throughout the paper we simply denote $\log_2$ as $\log$, $\log_e$ as $\ln$, and assume $\log 0 = - \infty$.
Assumption \ref{opt_oracle} is reasonable, as there are many classes of models where such an oracle exists.  For instance, polynomial time algorithms exist for finding the maximum weight matching in a weighted bipartite graph \cite{hopcroft1971n5,jonker1987shortest}. 
Graph cut algorithms can be applied to efficiently maximize a class of energy functions \cite{kolmogorov2004energy}.  
More generally, MAP inference can be performed efficiently for any log-supermodular weight function~\cite{orlin2009faster,chakrabarty2014provable,fujishige1980lexicographically}.  %A polynomial time algorithm was given by \citet{orlin2009faster}.  The Fujishige-Wolfe algorithm has a pseudopolynomial time guarantee, but better practical performance \cite{chakrabarty2014provable,fujishige1980lexicographically}.  
Our perturbation preserves the submodularity of $-\log w(\mathbf{x})$, as $\langle \mathbf{c}, \mathbf{x} \rangle$ can be viewed as $n$ independent single variable perturbations, so we have an efficient optimization oracle whenever the original weight function is log-supermodular.
Further, notice that this is a much weaker assumption compared with the optimization oracle required by randomized hashing methods \cite{chakraborty2013scalable,ermon2014low,zhao2016closing}.

If an approximate optimization oracle exists that can find a value within some known bound of the maximum, we can modify our bounds to use the approximate oracle.  This may improve the efficiency of our algorithm or extend its use to additional problem classes.  For the class of log-supermodular distributions, approximate MAP inference is equivalent to performing approximate submodular minimization \cite{jegelka2011fast}.

We note that even when an efficient optimization oracle exists, the problem of exactly calculating $Z(w)$ is generally still hard.  For example, polynomial time algorithms exist for finding the maximum weight perfect matching in a weighted bipartite graph.  However, computing the permanent of a bipartite graph's adjacency matrix, which equals the sum of weights for all perfect matchings or $Z(w)$, is still \#P-complete\cite{jerrum2004polynomial}.  A fully polynomial randomized
approximation scheme (FPRAS) exists \cite{jerrum2004polynomial,bezakova2006accelerating}, based on Markov chain Monte Carlo to sample over all perfect matchings.  However, the polynomial time complexity of this algorithm suffers from a large degree, limiting its practical use. 

\section{Weighted Rademacher Bounds on $Z(w)$}

Our approach for estimating the sum $Z(w) = \sum_{\mathbf{x}} w(\mathbf{x})$ is based on the idea that the Rademacher complexity of a set is related to the set's size. In particular, Rademacher complexity is monotonic in the sense that $R(A)\leq R(B)$ whenever $A \subseteq B$.  Note that monotonicity does not hold for $|A| \leq |B|$, that is, $R(A)$ is monotonic in the contents of $A$ but not necessarily in its size.  
We estimate the sum of arbitrary non-negative elements by generalizing the Rademacher complexity in definition \ref{weighted_rademacher}.
\begin{definition} \label{weighted_rademacher}
We define the weighted Rademacher complexity of a 
weight function
$w:\{-1,1\}^n \rightarrow [0, \infty)$ as
\begin{equation} \label{Delta_w}
\R(w) \coloneqq \E_{\mathbf{c}}\left[ \max_{\mathbf{x}\in \{-1,1\}^n} \{\langle \mathbf{c},\mathbf{x} \rangle + \log w(\mathbf{x})\} \right],
\end{equation}
for $\mathbf{c}$ sampled uniformly from $\{-1, 1\}^n$.
\end{definition}

In the notation of Eq.~(\ref{eq:randomOpt}), the weighted Rademacher complexity is simply $\R(w) = \E_\mathbf{c}[\delta(\mathbf{c}, w)]$.  
For a set $A \subseteq \{-1,1\}^n$, let $I_A:\{-1,1\}^n \rightarrow \{0, 1\}$
denote the indicator weight function for $A$, defined as $I_A(\mathbf{x}) = 1 \iff \mathbf{x} \in A$.  Then
$\R(I_A) = R(A)$, that is, the weighted Rademacher complexity is identical to the standard Rademacher complexity for indicator weight functions.  For a general weight function, the weighted Rademacher complexity extends the standard Rademacher complexity by giving each element (hypothesis) its own weight.

%-------------------------------------------------------------------------------------------------------------
\subsection{Algorithmic Strategy}

The key idea of this paper is to use the weighted Rademacher complexity $\R(w)$ to provide probabilistic estimates of $Z(w)$, the total weight of $w$.

This is a reasonable strategy because as we have seen before, for an indicator weight function $I_A:\{-1,1\}^n \rightarrow \{0, 1\}$, $\R(I_A)$ reduces to the standard Rademacher complexity $R(A)$, and $Z(I_A)=|A|$ is simply the cardinality of the set.
Therefore we can use known quantitative relationships between $R(A)$ and $|A|$ from learning theory to estimate $|A|=Z(I_A)$ in terms of $R(A)=\R(I_A)$.  Although not formulated in the framework of Rademacher complexity, this is the strategy used by 
\citet{barvinok1997approximate}.

Here, we generalize these results to general weight functions $w$ and show that it is, in fact, possible to use $\R(w)$ to obtain estimates of $Z(w)$. This observation can be turned into an algorithm by observing that $\R(w)$ is the expectation of a random variable concentrated around its mean. Therefore, as we will show in Proposition \ref{weighted_rademacher_slack}, a small number of samples suffices to reliably estimate $\R(w)$ (and hence, $Z(w)$) with high probability. Whenever $w$ is `sufficiently nice' and we have access to an optimization oracle, the estimation algorithm is efficient.

\begin{algorithm}[htb]
  \caption{Rademacher Estimate of $\log Z(w)$} \label{est_alg}
  \textbf{Inputs:} A positive integer $k$ and weight function $w:\{-1,1\}^n \rightarrow [0, \infty)$.
  
  \textbf{Output:} A number $\bar{\delta}_k(w)$ which approximates $\log Z(w)= \log \left( \sum_{\mathbf{x} \in \{-1,1\}^n} w(\mathbf{x})\right)$. 
  \begin{enumerate}
    \item Sample $k$ vectors $\mathbf{c}_1, \mathbf{c}_2, \dots, \mathbf{c}_k$ independently and uniformly from $\{-1,1\}^n$.

    \item
    Apply the optimization oracle of assumption \ref{opt_oracle} to each vector $\mathbf{c}$ and compute the mean
    \[
    \bar{\delta}_k(w) = \frac{1}{k} \sum_{i=1}^k \max_{\mathbf{x} \in \{-1,1\}^n} \{\langle \mathbf{c}_i,\mathbf{x} \rangle + \log w(\mathbf{x})\}.
    \]

  \item Output $\bar{\delta}_k(w)$ as an estimator of $\R(w)$ and thus $\log Z(w)$.
  \end{enumerate}
\end{algorithm}

\subsection{Bounding Weighted Rademacher Complexity}

The weighted Rademacher complexity is an expectation over optimization problems.  The optimization problem is defined by sampling a vector, or direction since all have length $\sqrt{n}$, uniformly from $\{-1,1\}^n$ and finding the vector $\mathbf{x}$ that is most aligned (largest dot product) after adding $\log w(\mathbf{x})$.  
Our first objective is to derive bounds on the weighted Rademacher complexity in terms of the sum $Z(w)$.

\begin{figure}[ht] 
\includegraphics[width=8cm]{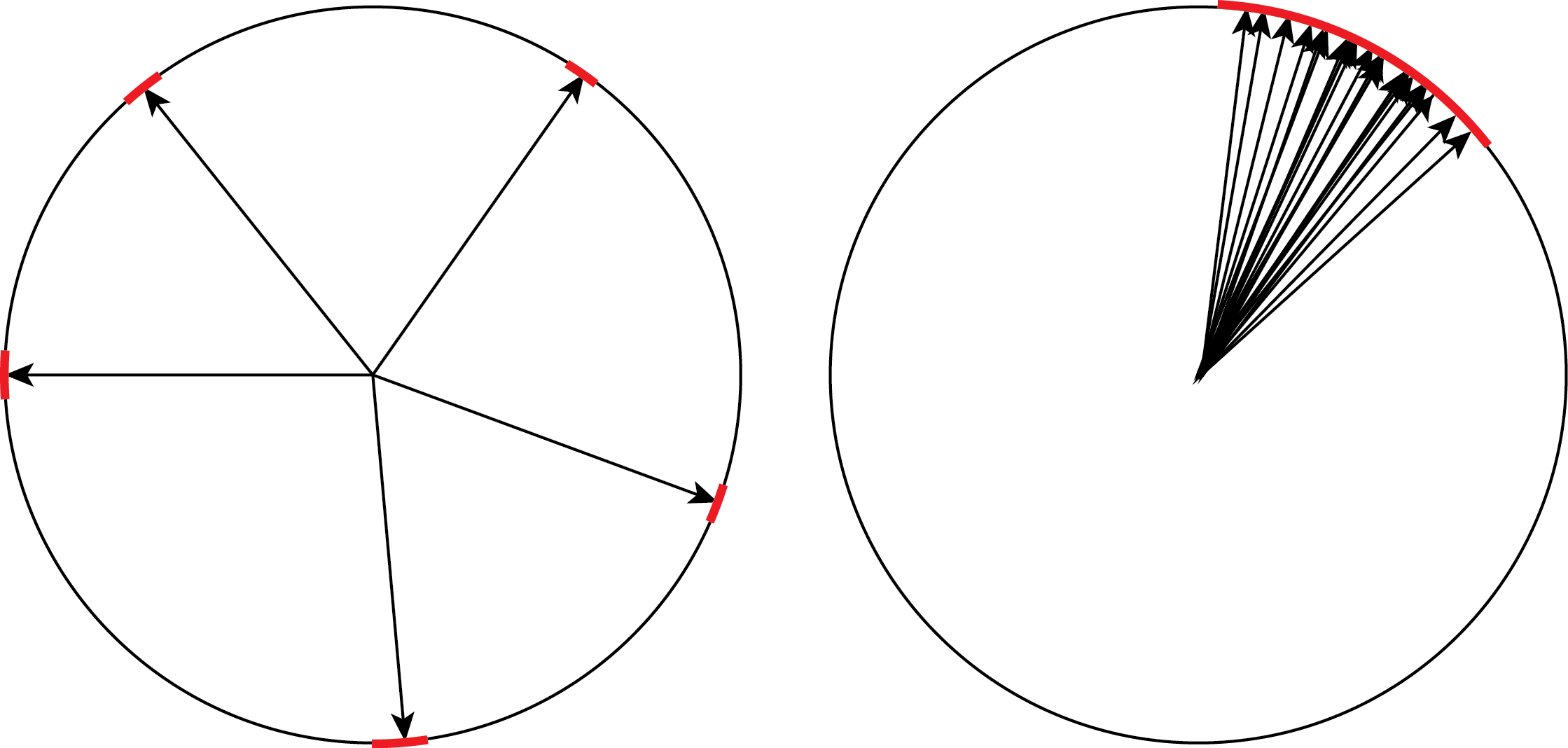}
\caption{Illustration mapping a set of vectors in high dimensional space $\{-1,1\}^n$ to the unit circle.  Red regions correspond to regions of space that have a large dot product with some vector in the set.  Left: when the size of a set is small, very few regions have a large dot product with any vector in the set, so the Rademacher complexity will be small.  Right: when a \emph{large set} of vectors is tightly packed in a \emph{small region} of space, the Rademacher complexity will remain relatively small.  
In both left and right figures we have similar (small) Rademacher complexities, yet different set sizes. This illustrates why tight bounds on the set size based on Rademacher complexity are difficult to achieve. 
} \label{bound_rademacher_vectors}
\end{figure}

We begin with the observation that it is impossible to derive bounds on the Rademacher complexity in terms of set size that are \emph{tight for sets of all shapes}.  To gain intuition, note that in high dimensional spaces the dot product of any particular vector and another chosen uniformly at random from $\{-1,1\}^n$ is close to 0 with high probability.  
The distribution of weight vectors throughout the space may take any geometric form.  One extreme configuration is that all vectors with large weights are packed tightly together, forming a Hamming ball.  At the other extreme, all vectors with large weights could be distributed uniformly through the space.  
%Our estimator $\R(w$) will under- or over-estimate, respectively, in these two cases.  Since the distribution of weight vectors is unknown, this places a fundamental limit on the tightness of our bounds.  See Figure \ref{bound_rademacher_vectors} for an illustration.
As Figure \ref{bound_rademacher_vectors} illustrates, a \emph{large} set of tightly packed vectors and a \emph{small} set of well-distributed vectors will both have similar Rademacher complexity. Thus, bounds on Rademacher complexity that are based on the underlying set's size fundamentally cannot always be tight for all distributions. Nevertheless, the lower and upper bounds we derive next are tight enough to be useful in practice.

%{\color{red}paragaph is confusing. not very connected with the previous one}
%Figure \ref{bound_rademacher_vectors} illustrates the difficulty of lower bounding the standard Rademacher complexity, where the best lower bound we can hope for is tight when given a set of closely packed vectors forming a hamming ball.  Lower bounding the weighted Rademacher complexity is similarly difficult.  The intuition follows the argument for the unweighted case, with the modification that we now are interested in the value $ \{\langle \mathbf{c},\mathbf{x} \rangle + \log w(\mathbf{x})\}$ for a particular vector $\mathbf{x} \in \{-1,1\}^n$ and another vector $\mathbf{c} \in \{-1,1\}^n$ sampled uniformly at random, instead of their dot product.  Now this value is close to $\log w(\mathbf{x})$ for most $\mathbf{c}$ and the weighted Rademacher complexity will remain close to $\log w_{max}$ (the largest weight) for a tightly packed set of vectors.
\subsubsection{Lower bound.}

To lower bound the weighted Rademacher complexity we adapt the technique of \cite{barvinok1997approximate} for lower bounding the standard Rademacher complexity.  The high level idea is that the space $\{-1,1\}^n$ can be mapped to the leaves of a binary tree. By following a path from the root to a leaf, we are dividing the space in half $n$ times, until we arrive at a leaf which corresponds to a single element (with some fixed weight).  By judiciously choosing which half of the space (branch of the tree) to recurse into at each step we derive the bound in Lemma \ref{lower_bound_weighted_rademacher}, whose proof is given in the appendix.  

\begin{lemma} \label{lower_bound_weighted_rademacher}
For any $\beta \in (0,1/2)$, the weighted Rademacher complexity of a weight function $w:\{-1,1\}^n \rightarrow [0, \infty)$ is lower bounded by
\small
\begin{equation*}
\R(w) \geq  \log w^*(\beta) +  \frac{ n \log \left(1-\beta\right)  + \log  Z(w)-\log w^*(\beta)}{\log \left(\frac{1 - \beta}{\beta}\right) },
\end{equation*}
\normalsize
with
\small
\begin{align*}
w^*(\beta) =
\begin{cases}
    w_{max} = \max_\mathbf{x} w(\mathbf{x}), & \text{if } \beta \geq 1/3\\
    w_{min} = \min_\mathbf{x} \{w(\mathbf{x}) : w(\mathbf{x}) > 0\}, & \text{if } \beta \leq 1/3\\
\end{cases}.
\end{align*}
\normalsize
\end{lemma}

\subsubsection{Upper bound.}
In the unweighted setting, a standard upper bound on the Rademacher complexity is used in learning theory to show that the Rademacher complexity of a small hypothesis class is also small, often to prove PAC-learnability.  Massart's Lemma (see \cite{shalev2014understanding}, lemma 26.8) formally upper bounds the Rademacher complexity in terms of the size of the set.  This result is intuitive since, as we have noted, the dot product between any one vector $\mathbf{x} \in \{-1,1\}^n$ is small with most other vectors $\mathbf{c} \in \{-1,1\}^n$.  Therefore, if the set is small the Rademacher complexity must also be small.  

Adapting the proof technique of Massart's Lemma to the weighted setting we arrive at
the following bound:

\begin{lemma} \label{upper_bound_weighted_rademacher}
For any $\lambda > 0$, $\gamma > 0$, and weight functions $w, w^\gamma:\{-1,1\}^n \rightarrow [0, \infty)$ with $w^\gamma(\mathbf{x}) = w(\mathbf{x})^\gamma$, the weighted Rademacher complexity of $w^\gamma$ is upper bounded by 
\begin{equation} 
\R(w^\gamma) \leq \frac{1}{\lambda} \log Z(w) +\frac{\lambda \gamma - 1}{\lambda} \log  w^*(\lambda, \gamma) +\lambda \frac{ n}{2},\\
\end{equation}
with
\small
\begin{align*} \label{w_star}
w^*(\lambda, \gamma) =
\begin{cases}
    w_{max} = \max_\mathbf{x} w(\mathbf{x}), & \text{if } \lambda \gamma \geq 1\\
    w_{min} = \min_\mathbf{x} \{w(\mathbf{x}) : w(\mathbf{x}) > 0\}, & \text{if } \lambda \gamma \leq 1\\
\end{cases}.
\end{align*}
\normalsize
\end{lemma}
%{\color{red} should $w^*$ also depend on $\gamma$}

Note that for an indicator weight function we recover the bound from Massart's Lemma by setting $\lambda = \sqrt{\frac{2 \log Z(w)}{n}}$ and $\gamma=1$.  

\begin{corollary} \label{lb_trivial}
%In the limit as $\gamma \rightarrow \infty$ by setting 
For sufficiently large $\gamma$ and
\[
\lambda = \sqrt{\frac{2 \log{\frac{Z(w)}{w_{max}}}}{n}},
\]
we recover the bound $w_{max} \leq Z(w)$ from Lemma \ref{upper_bound_weighted_rademacher}.
\end{corollary}

Lemma \ref{upper_bound_weighted_rademacher} holds for any $\lambda > 0$ and $\gamma  > 0$.  In general we set $\gamma=1$ and optimize over $\lambda$ to make the bound as tight as possible, comparing the result with the trivial bound given by Corollary \ref{lb_trivial}.  More sophisticated optimization strategies over $\lambda$ and $\gamma$ could result in a tighter bound.  Please see the appendix for further details and proofs.

\subsection{Bounding the Weighted Sum $Z(w)$}
With our bounds on the weighted Rademacher complexity from the previous section, we now present our method for efficiently bounding the sum $Z(w)$.  Proposition \ref{weighted_rademacher_slack} states that we can estimate the weighted Rademacher complexity using the optimization oracle of assumption \ref{opt_oracle}.

\begin{prop} \label{weighted_rademacher_slack}
For $\mathbf{c} \in \{-1,1\}^n$ sampled uniformly at random, the bound 
\begin{equation}
\R(w) - \sqrt{6 n} \leq \delta(\mathbf{c}, w) \leq \R(w) + \sqrt{6 n}
\end{equation}
holds with probability greater than .95.
\end{prop}

\begin{proof}
By applying Proposition \ref{mcdiarmid} to the function $f_w(\mathbf{c}) = \delta(\mathbf{c}, w)$, and noting the constant $d_i = 2$, we have
\begin{align*}
P\left[|\delta(\mathbf{c}, w) - \R(w)| \geq \sqrt{6 n} \right] \leq e^{-3} \leq .05.
\end{align*}
This finishes the proof. 
\end{proof}

To bound $Z(w)$ we use our optimization oracle to solve a perturbed optimization problem, giving an estimate of the weighted Rademacher complexity, $\R(w)$. Next we invert the bounds on $\R(w)$ (Lemmas \ref{lower_bound_weighted_rademacher} and \ref{upper_bound_weighted_rademacher}) to obtain bounds on $Z(w)$.  We optimize the parameters $\lambda$ and $\beta$ (from equations \ref{lower_bound_weighted_rademacher} and \ref{upper_bound_weighted_rademacher}) to make the bounds as tight as possible.  By applying our optimization oracle repeatedly, we can reduce the slack introduced in our final bound when estimating $\R(w)$ (by Lemma \ref{weighted_rademacher_slack}) and arrive at our bounds on the sum $Z(w)$, stated in the following theorem.

\begin{theorem}\label{main_result}  With probability at least $0.95$, the sum $Z(w)= \sum_{\mathbf{x} \in \{-1,1\}^n} w(\mathbf{x})$ of any weight function $w:\{-1,1\}^n \rightarrow [0, \infty)$ is bounded by the outputs of  algorithms \ref{lower_alg} and \ref{upper_alg} as
\[
\psi_{LB} < \log Z(w)< \psi_{UB}.
\]
\end{theorem}

\begin{algorithm}[htb]
  \caption{Rademacher Lower Bound for $\log Z(w)$} \label{lower_alg}
  \textbf{Inputs:} The estimator $\bar{\delta}_k(w)$ output by algorithm \ref{est_alg}, $k$ used to compute $\bar{\delta}_k(w)$, and optionally $w_{min}$ and $w_{max}$.
  
  \textbf{Output:} A number $\psi_{LB}$ which lower bounds $\log Z(w)$.
    \begin{enumerate}
      \item
      If $\log w_{min}$ was provided as input, calculate 
	\[
      \lambda = \frac{\bar{\delta}_k(w) - \sqrt{\frac{6 n}{k}} - \log w_{min}}{n}
	\]
    
	  \item
      If $\log w_{min}$ was provided as input and $\lambda \leq 1$,
      \[
		\psi_{LB} = \frac{(\bar{\delta}_k(w) - \sqrt{\frac{6 n}{k}} - \log w_{min})^2}{2n} + \log w_{min}.
	  \]
    
	  \item
      Otherwise,
      \[
		\psi_{LB} = \bar{\delta}_k(w) - \sqrt{\frac{6 n}{k}} - \frac{n}{2}.
	  \]
      
      \item Output the lower bound $\max\{\psi_{LB}, \log w_{max}\}$. 
    \end{enumerate}
	  
\end{algorithm}

\begin{algorithm}[htb]
  \caption{Rademacher Upper Bound for $\log Z(w)$} \label{upper_alg}
  \textbf{Inputs:} The estimator $\bar{\delta}_k(w)$, $k$ used to compute $\bar{\delta}_k(w)$, and optionally $w_{min}$ and $w_{max}$.
  
  \textbf{Output:} A number $\psi_{UB}$ which upper bounds $\log Z(w)$.
  
  \begin{enumerate}
      \item If $w_{min}$ was provided as input, calculate 
      \[
    \beta_{min} = \frac{\bar{\delta}_k(w) + \sqrt{\frac{6 n}{k}} - \log w_{min}}{n}.
    \]      
      \item If $w_{max}$ was provided as input, calculate 
      \[
    \beta_{max} = \frac{\bar{\delta}_k(w) + \sqrt{\frac{6 n}{k}} - \log w_{max}}{n}.
    \]
      \item Set the value
\begin{align*}
\beta_{opt} = 
\begin{cases}
  \beta_{min}, & \text{if } 0 < \beta_{min} < \frac{1}{3} \\  
  \beta_{max}, & \text{if } \frac{1}{3} < \beta_{max} < \frac{1}{2} \\  
  \frac{1}{2}, & \text{if } \frac{1}{2} < \beta_{max} \\ 
  \frac{1}{3}, & \text{otherwise} \\
\end{cases}.
\end{align*}

    \item Output the upper bound $\psi_{UB}$:
      \begin{enumerate}
        \item If $\beta_{opt} = \frac{1}{3}$, $\psi_{UB} = \bar{\delta}_k(w) + \sqrt{\frac{6 n}{k}} + n \log \left( \frac{3}{2} \right)$.
      
        \item If $\beta_{opt} = \frac{1}{2}$, $\psi_{UB} =  n + \log w_{max}$.

        \item Otherwise,
        
        \footnotesize
        \begin{equation*}
        \psi_{UB} =  n \beta_{opt} \log  \left( \frac{1 - \beta_{opt}}{\beta_{opt}}\right) - n \log \left( 1 - \beta_{opt} \right) + \log w^*,
        \end{equation*}
        \normalsize
        
        where $w^* =
    	\begin{cases}
  			w_{min}, & \text{if } \beta_{opt} < \frac{1}{3} \\  
  			w_{max}, & \text{if } \beta_{opt} > \frac{1}{3} \\  
     	\end{cases}$.
      \end{enumerate}      
  \end{enumerate}
\end{algorithm}

%-------------------------------------------------------------------------------------------------------------
\section{Experiments}

The closest line of work to this paper showed that the partition function can be bounded by solving an optimization problem perturbed by Gumbel random variables \cite{hazan2012partition,hazan2013sampling,hazan2016high,kim2016exact,pmlr-v70-balog17a}.  This approach is based on the fact that 
\[
\ln Z(w) = \E_\gamma \left[ \max_{\mathbf{x} \in \{-1,1\}^n} \left\{ \ln w(\mathbf{x}) + \gamma(\mathbf{x}) \right\} \right],
\]
where \emph{all} $2^n$ random variables $\gamma(\mathbf{x})$ are sampled from the Gumbel distribution with scale $1$ and shifted by the Euler-Mascheroni constant to have mean 0.  Perturbing all $2^n$ states with IID Gumbel random variables is intractable, leading the authors to bound $\ln Z(w)$ by perturbing states with a combination of low dimensional Gumbel perturbations.  Specifically the upper bound
\small
\[
\ln Z(w) \leq \Theta_{UB} = \E_\gamma \left[ \max_{\mathbf{x} \in \{-1,1\}^n} \left\{ \ln w(\mathbf{x}) + \sum_{i=1}^n \gamma_i(x_i) \right\} \right]
\]
\normalsize
\cite{hazan2016high} and lower bound
\small
\[
\ln Z(w) \geq \Theta_{LB} = \E_\gamma \left[ \max_{\mathbf{x} \in \{-1,1\}^n} \left\{ \ln w(\mathbf{x}) + \sum_{i=1}^n \frac{1}{n}\gamma_i(x_i) \right\} \right]
\]
\normalsize
\cite[p.~6]{pmlr-v70-balog17a} hold in expectation, where $\gamma_i(x)$ for $i =1,\dots,n$ are sampled from the Gumbel distribution with scale $1$ and shifted by the Euler-Mascheroni constant to have mean 0. 

To obtain bounds that hold with high probability using Gumbel perturbations we calculate the slack term \cite[p.~32]{hazan2016high} 
\scriptsize
\begin{align*}
\epsilon_g = \min & \left\{ 2 \sqrt{n} \left( 1 + \sqrt{\frac{1}{2 k} \ln \frac{2}{\alpha}} \right)^2, \sqrt{n} \max \left\{\frac{4}{k}\ln \frac{2}{\alpha} , \sqrt{ \frac{32}{k} \ln \frac{2}{\alpha} } \right\} \right\} \\
\end{align*}
\normalsize
giving upper and lower bounds $\theta_{UB} = \Theta_{UB} + \epsilon_g$ and $\theta_{LB} = \Theta_{LB} - \frac{\epsilon_g}{n}$ that hold with probability $1-\alpha$ where $k$ samples are used to estimate the expectation bounds.

We note the Gumbel expectation upper bound takes nearly the same form as the weighted Rademacher complexity, with two differences.  The perturbation is sampled from a Gumbel distribution instead of a dot product with a vector of Rademacher random variables and, without scaling, the two bounds are naturally written in different log bases.

We experimentally compare our bounds with those obtained by Gumbel perturbations on two models.  First we bound the partition function of the spin glass model from \cite{hazan2016high}.  For this problem the weight function is given by the unnormalized probability distribution of the spin glass model.  Second we bound the propositional model counts (\#SAT) for a variety of SAT problems.  This problem falls into the unweighted category where every weight is either 0 or 1, specifically every satisfying assignment has weight 1 and we bound the total number of satisfying assignments.

\begin{figure}[ht] 
\includegraphics[width=8cm]{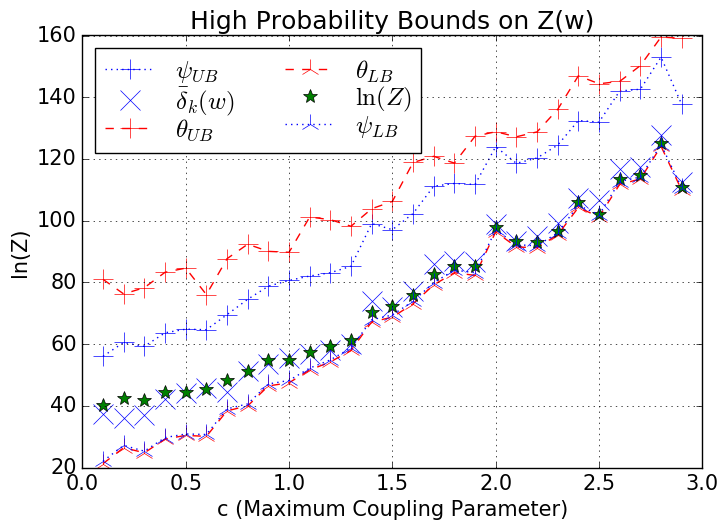}
\caption{Bounds for a 7x7 spin glass model with $k=5$ (for both methods), that hold with probability .95.  Our bounds and estimator are scaled to match Gumbel log base $e$ bounds.} \label{high_prob_ising} 
\end{figure}

\subsection{Spin Glass Model}
Following \cite{hazan2016high}, we bound the partition function of a spin glass model with variables $x_i \in \{-1,1\}$ for $i=1,2,\dots,n$, where each variable represents a spin.  Each spin has a local field parameter $\theta_i$ which corresponds to its local potential function $\theta_i(x_i) = \theta_i x_i$.  We performed experiments on grid shaped models where each spin variable has 4 neighbors, unless it occupies a grid edge.  Neighboring spins interact with coupling parameters $\theta_{i,j}(x_i, x_j) = \theta_{i,j} x_i x_j$.  The potential function of the spin glass model is
\[
\theta(x_1, x_2, \dots, x_n) = \sum_{i \in V} \theta_i x_i + \sum_{(i,j) \in E} \theta_{i,j} x_i x_j,
\]
with corresponding weight function
\[
w(\mathbf{x}) = \exp \left( \sum_{i \in V} \theta_i x_i + \sum_{(i,j) \in E} \theta_{i,j} x_i x_j \right).
\]
We compare our bounds on a 7x7 spin glass model.  We sampled the local field parameters $\theta_i$ uniformly at random from $[-1,1]$ and the coupling parameters uniformly at random from $[0,c)$ with $c$ varying.  Non-negative coupling parameters make it possible to perform MAP inference efficiently using the graph-cuts algorithm \cite{kolmogorov2004energy,greig1989exact}.  We used the python maxflow module wrapping the implementation from \citet{boykov2004experimental}.

Figure \ref{high_prob_ising} shows bounds that hold with probability .95, where all bounds are computed with $k=5$.  For this value of $k$, our approach produces tighter upper bounds than using Gumbel perturbations.  The crossover to a tighter Gumbel perturbation upper bound occurs around $k \approx 15$.  Lower bounds are equivalent, although we note it is trivial to recover this bound by simply calculating the largest weight over all states.

\begin{table*}[htb] 
\title{Test title}
\centering
    \small
    \begin{tabular}{@{}lcc|cc|cc|cc@{}}
    \toprule
Model Name & \#Variables& \#Clauses & ln(Z)& $\bar{\delta}_1(w)$ &  $\psi_{UB}$ &  $\theta_{UB}$ &  $\psi_{LB}$ &  $\theta_{LB}$ \\\midrule
     log-1             & 939                  & 3785                & 47.8                   & 64.5        (20.8)          & 438.0       (46.2)     & \textbf{426.5}        (43.0)       & \textbf{0.5}     (0.6)       & -0.3      (0.0)     \\ 
           log-2             & 1337                  & 24777                & 24.2                   & 48.6        (20.7)          & 485.7       (60.3)     & \textbf{464.0}        (45.1)       & \textbf{0.3}     (0.4)       & -0.3      (0.0)     \\ 
           log-3             & 1413                  & 29487                & 26.4                   & 49.9        (22.3)          & 503.9       (65.3)     & \textbf{478.2}        (42.3)       & \textbf{0.4}     (0.4)       & -0.3      (0.0)     \\ 
           log-4             & 2303                  & 20963                & 65.3                   & 106.0        (26.6)          & 830.2       (77.7)     & \textbf{676.9}        (58.8)       & \textbf{0.4}     (0.5)       & -0.2      (0.0)     \\ 
           tire-1             & 352                  & 1038                & 20.4                   & 30.7        (11.2)          & \textbf{198.5}      (17.6)      & 249.6        (23.7)       & \textbf{0.3}     (0.4)       & -0.5      (0.1)     \\ 
           tire-2             & 550                  & 2001                & 27.3                   & 42.1        (14.2)          & \textbf{283.9}      (27.7)      & 310.2        (29.6)       & \textbf{0.3}     (0.4)       & -0.4      (0.1)     \\ 
           tire-3             & 577                  & 2004                & 26.1                   & 36.9        (17.1)          & \textbf{280.5}      (36.1)      & 316.5        (29.1)       & \textbf{0.4}     (0.6)       & -0.4      (0.1)     \\ 
           tire-4             & 812                  & 3222                & 32.3                   & 55.0        (17.4)          & 384.7       (38.9)     & \textbf{383.3}        (35.3)       & \textbf{0.3}     (0.3)       & -0.3      (0.0)     \\ 
           ra             & 1236                  & 11416                & 659.2                   & 621.1        (15.5)          & \textbf{856.7}      (0.0)      & 1100.9        (45.7)       & \textbf{184.1}     (10.1)       & 0.3      (0.0)     \\ 
           rb             & 1854                  & 11324                & 855.9                   & 857.2        (12.6)          & \textbf{1285.1}      (0.0)      & 1387.5        (43.9)       & \textbf{239.3}     (7.7)       & 0.2      (0.0)     \\ 
           sat-grid-pbl-0010             & 110                  & 191                & 54.7                   & 51.6        (4.9)          & \textbf{76.2}      (0.0)      & 176.3        (13.6)       & \textbf{7.6}     (2.2)       & -0.5      (0.1)     \\ 
           sat-grid-pbl-0015             & 240                  & 436                & 125.4                   & 120.2        (6.6)          & \textbf{166.4}      (0.0)      & 310.3        (18.8)       & \textbf{26.6}     (3.8)       & -0.1      (0.1)     \\ 
           sat-grid-pbl-0020             & 420                  & 781                & 220.4                   & 215.5        (9.0)          & \textbf{291.1}      (0.0)      & 472.3        (26.5)       & \textbf{56.2}     (5.6)       & 0.0      (0.1)     \\ 
           sat-grid-pbl-0025             & 650                  & 1226                & 348.3                   & 338.8        (9.4)          & \textbf{450.5}      (0.0)      & 667.8        (33.1)       & \textbf{97.0}     (6.2)       & 0.2      (0.1)     \\ 
           sat-grid-pbl-0030             & 930                  & 1771                & 502.4                   & 482.6        (13.0)          & \textbf{644.6}      (0.0)      & 893.3        (36.7)       & \textbf{144.1}     (8.7)       & 0.2      (0.0)     \\ 
           c432             & 196                  & 514                & 25.0                   & 42.7        (5.8)          & \textbf{135.3}      (1.0)      & 212.6        (18.2)       & \textbf{1.4}     (0.8)       & -0.5      (0.1)     \\ 
           c499             & 243                  & 714                & 28.4                   & 58.5        (6.2)          & \textbf{168.2}      (0.4)      & 243.8        (16.6)       & \textbf{3.2}     (1.2)       & -0.4      (0.1)     \\ 
           c880             & 417                  & 1060                & 41.6                   & 83.1        (8.4)          & \textbf{281.1}      (4.7)      & 332.9        (21.3)       & \textbf{4.2}     (1.4)       & -0.3      (0.1)     \\ 
           c1355             & 555                  & 1546                & 28.4                   & 79.8        (12.2)          & \textbf{342.9}      (14.2)      & 368.6        (28.7)       & \textbf{2.3}     (1.3)       & -0.3      (0.1)     \\ 
           c1908             & 751                  & 2053                & 22.9                   & 87.8        (12.2)          & 427.7       (19.3)     & \textbf{419.1}        (32.8)       & \textbf{1.8}     (0.9)       & -0.3      (0.0)     \\ 
           c2670             & 1230                  & 2876                & 161.5                   & 260.0        (14.6)          & 812.8       (10.8)     & \textbf{701.4}        (39.6)       & \textbf{23.7}     (3.4)       & -0.1      (0.0)     \\ 
 \bottomrule
    \end{tabular}
\caption{Empirical comparison of our estimate of ($\bar{\delta}_1(w)$) and bounds on ($\psi$) propositional model counts against bounds based on Gumbel perturbations ($\theta$).  The mean over 100 runs is shown with the standard deviation in parentheses.  Bounds hold with probability .95 and $k=1$ for both methods.  Tighter bounds are in \textbf{bold}.  Meta column descriptions, left to right: model name and information, natural logarithm of ground truth model counts and our estimator, upper bounds, and lower bounds.} \label{high_prob_sat}
\end{table*}

\subsection{Propositional Model Counting}
Next we evaluate our method on the problem of propositional model counting.  Given a boolean formula $F$, this poses the question of how many assignments $\mathbf{x}$ to the underlying boolean variables result in $F$ evaluating to true.  Our weight function is given by $w(\mathbf{x}) = 1$ if $F(\mathbf{x})$ evaluates to true, and $0$ otherwise.

We performed MAP inference on the perturbed problem using the weighted partial MaxSAT solver MaxHS \cite{davies2013solving}.  Ground truth was obtained for a variety of models\footnote{The models used in our experiments can be downloaded from http://reasoning.cs.ucla.edu/c2d/results.html} using three exact propositional model counters \cite{thurley2006sharpsat,sang2004combining,oztok2015top}\footnote{Precomputed model counts were downloaded from https://sites.google.com/site/marcthurley/sharpsat/benchmarks/ collected-model-counts}.  Table \ref{high_prob_sat} shows bounds that hold with probability .95 and $k = 1$.  While the Gumbel lower bounds are always trivial, we produce non-trivial lower bounds for several model instances.  Our upper bounds are generally comparable to or tighter than Gumbel upper bounds.

\subsection{Analysis}

Our bounds are much looser than those computed by randomized hashing schemes \citep{chakraborty2013scalable,wishicml13,ermon2013embed,zhao2016closing}, but also require much less computation~\cite{uai13LPCount,achim2016beyond}.  While our approach provides polynomial runtime guarantees for MAP inference in the spin glass model after random perturbations have been applied, randomized hashing approaches do not.  For propositional model counting, we found that our method is computationally cheaper by over 2 orders of magnitude than results reported in \citet{zhao2016closing}.  Additionally, we tried reducing the runtime and accuracy of randomized hashing schemes by running code from \citet{zhao2016closing} with $f$ values of 0, .01, .02, .03, .04, and .05.  We set the maximum time limit to 1 hour (while our method required .01 to 6 seconds of computation for reported results).  Throughout experiments on models reported in Table \ref{high_prob_sat} our approach still generally required orders of magnitude less computation and also found tighter bounds in some instances.

Empirically, our lower bounds were comparable to or tighter than those obtained by Gumbel perturbations on both models.  The weighted Rademacher complexity is generally at least as good an estimator of $\log Z$ as the Gumbel upper bound, however it is only an estimator and not an upper bound.  Our upper bound using the weighted Rademacher complexity, which holds in expectation, is empirically weaker than the corresponding Gumbel expectation upper bound.  However, the slack term needed to transform our expectation bound into a high probability bound is tighter than the corresponding Gumbel slack term.  Since both slack terms approach $0$ in the limit of infinite computation ($k=\infty$, the number of samples used to estimate the expectation bound), this can result in a trade-off where we produce a tighter upper bound up to some value of $k$, after which the Gumbel bound becomes tighter.

%-------------------------------------------------------------------------------------------------------------
\section{Conclusion}
We introduced the weighted Rademacher complexity, a novel generalization of Rademacher complexity.  We showed that this quantity can be used as an estimator of the size of a weighted set, and gave bounds on the weighted Rademacher complexity in terms of the weighted set size.  This allowed us to bound the sum of any non-negative weight function, such as the partition function, in terms of the weighted Rademacher complexity.  We showed how the weighted Rademacher complexity can be efficiently approximated whenever an efficient optimization oracle exists, as is the case for a variety of practical problems including calculating the partition function of certain graphical models
and the permanent of non-negative matrices.  Experimental evaluation demonstrated that our approach provides tighter bounds than competing methods under certain conditions.

In future work our estimator $\R(w)$ and bounds on $Z(w)$ may be generalized to other forms of randomness.  Rather than sampling $\mathbf{c}$ uniformly from $\{-1, 1\}^n$, we could conceivably sample each element $c_i$ from some other distribution, such as the uniform distribution over $[-1,1]$, a Gaussian, or Gumbel.  Our bounds should readily adapt to continuous uniform or gaussian distributions, although derivations may be more complex in general.  As another line of future work, the weighted Rademacher complexity may be useful beyond approximate inference to learning theory. 

\section{Acknowledgments}
We gratefully acknowledge funding from Ford,  FLI and NSF grants $\#1651565$, $\#1522054$, $\#1733686$.  We also thank Tri Dao, Aditya Grover, Rachel Luo, and anonymous reviewers.

%\fontsize{9.0pt}{10.0pt} \selectfont
\bibliography{bibliography}
\bibliographystyle{aaai}

\clearpage
\onecolumn
\section{Appendix}
We present formal proofs of our bounds on the sum $Z(w)$ of any non-negative weight function $w:\{-1,1\}^n \rightarrow [0, \infty)$.  For readability we occasionally restate results from the main paper.  The format of our proof is as follows.  First we bound the weighted Rademacher complexity, $\R(w)$, by the output of our optimization oracle ($\delta(\mathbf{c}, w)$, described in Assumption~\ref{opt_oracle}), which we refer to as the \textbf{slack bound}.  Next we \textbf{lower bound} the sum $Z(w)$ by $\R(w)$ and apply our slack bound to obtain a lower bound on $Z(w)$ in terms of $\delta(\mathbf{c}, w)$.  Similarly, we \textbf{upper bound} the sum $Z(w)$ by $\R(w)$ and apply our slack bound to obtain an upper bound on $Z(w)$ in terms of $\delta(\mathbf{c}, w)$.  Finally we \textbf{tighten the bounds} by repeatedly applying our optimization oracle.  %Note that throughout the appendix, as in the paper, we simply denote $\log_2$ as $\log_2$, $\log_2_e$ as $\ln$, and assume $\log_2 0 = - \infty$.

\subsection{Slack Bound}

We use McDiarmid's bound (Proposition \ref{mcdiarmid}) to bound the difference between the output of our optimization oracle ($\delta(\mathbf{c}, w)$, described in Assumption~\ref{opt_oracle}) and its expectation, which is the weighted Rademacher complexity $\R (w)$.  For the function
\[f_w(\mathbf{c}) = \delta(\mathbf{c},w) = \max_{\mathbf{x} \in \{-1,1\}^n} \{\langle \mathbf{c},\mathbf{x} \rangle + \log_2 w(\mathbf{x})\},\] 
the constant $d_i = 2$ in McDiarmid's bound, giving
\begin{align*}
P[|\delta(\mathbf{c}, w) - \R(w)|\geq\epsilon] \leq \exp \left( \frac{-2 \epsilon^2}{\sum_j c_j^2} \right) \\
P[|\delta(\mathbf{c}, w) - \R(w)|\geq\sqrt{6 n}] \leq \exp \left( \frac{-2 (\sqrt{6 n})^2}{4n} \right) \leq .05.\\
\end{align*}
By choosing $\mathbf{c} \in \{-1,1\}^n$ uniformly at random we can say with probability greater than .95 that
\begin{equation} \label{bound_max_by_expectation}
\R(w) - \sqrt{6 n} \leq \delta(\mathbf{c}, w) \leq \R(w) + \sqrt{6 n}.
\end{equation}

\subsection{Lower Bound} \label{appendix_LB}
In this section we lower bound the sum $Z(w)$ by $\R(w)$ and apply our slack bound to obtain a lower bound on $Z(w)$ in terms of $\delta(\mathbf{c}, w)$.  We extend the Massart lemma \cite[lemma 26.8]{shalev2014understanding} to the weighted setting by accounting for the $\log_2 w(\mathbf{x})$ weight term in the weighted Rademacher complexity.  Our lower bound on $Z(w)$ is given by the following Lemma:

\begin{lemma}For $\mathbf{c} \in \{-1,1\}^n$ sampled uniformly at random, the following bound holds with probability greater than .95:
% Let A = {a1, . . . , aN } be a finite set of vectors
% in R
% m. Define a¯ =
% 1
% N
% PN
% i=1 ai. Then,
% R(A) ≤ max
% a∈A
% ka − a¯k
% p
% 2 log(N)
% m
\begin{align*}
\log_2 Z(w) \geq
\begin{cases}  
	\frac{(\delta(\mathbf{c}, w) - \sqrt{6 n} - \log_2 w_{min})^2}{2n} + \log_2 w_{min}, & \text{if } \frac{\delta(\mathbf{c}, w) - \sqrt{6 n} - \log_2 w_{min}}{n} \leq 1\\  
	\delta(\mathbf{c}, w) - \sqrt{6 n} - \frac{n}{2}, & \text{otherwise}
\end{cases}
\end{align*}
\end{lemma}
\begin{proof}
We begin by upper bounding $\R(w)$ in terms of $Z(w)$.  Define $\mathbf{c} \in \{-1 ,1\}^n$ generated uniformly at random and $\mathbf{x} \in \{-1 ,1\}^n$.  For any $\lambda > 0$, $\gamma > 0$, and weight functions $w, w^\gamma:\{-1,1\}^n \rightarrow [0, \infty)$ with $w^\gamma(\mathbf{x}) = w(\mathbf{x})^\gamma$ we have
\begin{align*} 
\R(w^\gamma) = \E_{\mathbf{c}} \left[\max_\mathbf{x} \langle \mathbf{c},\mathbf{x} \rangle + \gamma \log_2 w(\mathbf{x})\right] =\\
\frac{1}{\lambda} \E_{\mathbf{c}} \left[\max_\mathbf{x} \lambda \langle \mathbf{c},\mathbf{x} \rangle + \lambda \gamma \log_2 w(\mathbf{x})\right] = \frac{1}{\lambda} \E_{\mathbf{c}} \left[\log_2 \max_\mathbf{x} 2^{\lambda (\langle \mathbf{c},\mathbf{x} \rangle + \gamma \log_2 w(\mathbf{x}))}\right] \leq \\
\frac{1}{\lambda} \E_{\mathbf{c}} \left[\log_2 \sum_\mathbf{x} 2^{\lambda (\langle \mathbf{c},\mathbf{x} \rangle + \gamma \log_2 w(\mathbf{x}))}\right] \stackrel{\text{Jensen}}{\leq} \\
\frac{1}{\lambda} \log_2 \E_{\mathbf{c}} \left[\sum_{\mathbf{x}} 2^{\lambda(\langle \mathbf{c},\mathbf{x} \rangle + \gamma \log_2 w(\mathbf{x}))}\right], \\
\end{align*}
where we have used Jensen's inequality.  By the linearity of expectation and independence between elements $c_i$ in a random vector $\mathbf{c}$,
\begin{align*} 
\R(w^\gamma) \leq \frac{1}{\lambda} \log_2 \sum_{\mathbf{x}} \left( 2^{\lambda \gamma \log_2 w(\mathbf{x})} \E_{\mathbf{c}} \left[2^{\lambda \langle \mathbf{c},\mathbf{x} \rangle}\right] \right) =\frac{1}{\lambda} \log_2 \sum_{\mathbf{x}} \left( 2^{\lambda \gamma \log_2 w(\mathbf{x})} \prod_{i=1}^n \E_{c_i} \left[2^{\lambda c_i x_i}\right] \right). \\
\end{align*}
Using Lemma A.6 from \cite{shalev2014understanding},
\begin{align*}
\R(w^\gamma) \leq \frac{1}{\lambda} \log_2 \sum_{\mathbf{x}} \left( 2^{\lambda \gamma \log_2 w(\mathbf{x})} \prod_{i=1}^n \frac{2^{ \lambda x_i}+2^{-\lambda x_i}}{2} \right) \leq \frac{1}{\lambda}  \log_2 \sum_{\mathbf{x}} \left( 2^{\lambda \gamma \log_2 w(\mathbf{x})} \prod_{i=1}^n 2^{ \frac{(\lambda x_i)^2}{2}} \right) = \\
\frac{1}{\lambda} \log_2 \sum_{\mathbf{x}} \left( 2^{\lambda \gamma \log_2 w(\mathbf{x})} 2^{ \frac{\lambda^2 ||\mathbf{x}||^2}{2}} \right) = \frac{1}{\lambda} \log_2 \left( 2^{\frac{\lambda^2 n}{2}} \sum_{\mathbf{x}} 2^{\lambda \gamma \log_2 w(\mathbf{x})} \right) = \frac{1}{\lambda} \log_2 \left( \sum_{\mathbf{x}} 2^{\lambda \gamma \log_2 w(\mathbf{x})} \right) + \frac{\lambda n}{2}. \\
\end{align*}
\jonathan{The inequality $(2^{-\lambda} + 2^\lambda)/2 \leq 2^{\lambda^2/2}$ is good for $\lambda<1$, but poor for $\lambda>>1$.} Next,
\begin{align} \label{rademacher_upperBound}
\begin{split}
\R(w^\gamma) \leq
\frac{1}{\lambda} \log_2 \left(\sum_{\mathbf{x}} w(\mathbf{x})^{\lambda\gamma} \right) + \frac{\lambda n}{2} 
=  \frac{1}{\lambda} \log_2 \left(\sum_{\mathbf{x}} w(\mathbf{x}) w(\mathbf{x})^{\lambda\gamma-1} \right) + \frac{\lambda n}{2} \leq\\
\frac{1}{\lambda} \log_2 \left(\max_\mathbf{x} \left\{ w(\mathbf{x})^{\lambda\gamma-1} \right\} \sum_{\mathbf{x}} w(\mathbf{x})  \right) + \frac{\lambda n}{2}  =
\frac{1}{\lambda} \log_2 Z(w)  +\frac{1}{\lambda} \log_2  \max_\mathbf{x} \{w_{max}^{\lambda\gamma-1},w_{min}^{\lambda\gamma-1}\} +\lambda \frac{ n}{2} =\\
\frac{1}{\lambda} \log_2 Z(w)  +\frac{\lambda\gamma - 1}{\lambda} \log_2  w^*(\lambda, \gamma) +\lambda \frac{ n}{2},\\
\end{split}
\end{align}
where
\begin{align*} \label{w_star}
w^*(\lambda, \gamma) =
\begin{cases}
    w_{max} = \max_\mathbf{x} w(\mathbf{x}), & \text{if } \lambda \gamma \geq 1\\
    w_{min} = \min_\mathbf{x} \{w(\mathbf{x}) : w(\mathbf{x}) > 0\}, & \text{if } \lambda \gamma \leq 1\\
\end{cases}.
\end{align*}
Note that for $\lambda \gamma = 1$ we have two valid inequalities that hold for either choice of $w^*(\lambda, \gamma)$.\jonathan{can the notation be improved for defining $w^*(\lambda)$ when $\lambda = 1$?}  Having bounded the weighted Rademacher complexity in terms of $Z(w)$, we now apply the slack bound from equation \ref{bound_max_by_expectation} and have that with probability greater than .95
\begin{equation} \label{delta_UB}
\delta(\mathbf{c}, w) \leq \frac{1}{\lambda} \log_2 Z(w)  +\frac{\lambda \gamma - 1}{\lambda} \log_2  w^*(\lambda, \gamma) +\lambda \frac{ n}{2} + \sqrt{6 n}.\\
\end{equation}
This upper bound on $\delta(w)$ holds for any $\lambda > 0$, so we could jointly optimize over $\lambda$ and $\gamma$ to make the bound as tight as possible.  However, this is non-trivial because changing $\gamma$ changes the weight function we supply to our optimization oracle.  Instead we generally set $\gamma=1$ and optimize over only $\lambda$.  At the end of this section we derive another bound with a different choice of $\gamma$.  This bound is trivial to derive, but illustrates that other choices of $\gamma$ could result in meaningful bounds.

Rewriting the bound in Equation~\ref{delta_UB} with $\gamma=1$ and $w^*(\lambda) = w^*(\lambda, 1)$ we have
\begin{equation} \label{lb_all_one_side}
-(\lambda - 1) \log_2  w^*(\lambda) -\lambda^2 \frac{ n}{2} + \lambda (\delta(\mathbf{c}, w) - \sqrt{6 n}) \leq  \log_2 Z(w)  ,
\end{equation}
so the optimal value of $\lambda$ that makes our bound as tight as possible occurs at the maximum of the quadratic function
\begin{align*}
h(\lambda) = -\lambda^2 \frac{n}{2} - \lambda (\log_2 w^*(\lambda) - (\delta(\mathbf{c}, w) - \sqrt{6 n})) + \log_2 w^*(\lambda)\\
h'(\lambda) = - \lambda n - \log_2 w^*(\lambda) + \left( \delta(\mathbf{c}, w) - \sqrt{6 n} \right) \\
h''(\lambda) = -n,
\end{align*}
Where the stated derivatives are valid for $\lambda \neq 1$, as $w^*(\lambda)$ is piecewise constant with a discontinuity at $\lambda = 1$.  The maximum of $h(\lambda)$ must occur at $\lambda = -\infty, +\infty,1$, or the value of $\lambda$ that makes $h'(\lambda)=0$.  By inspection the maximum does not occur at $\lambda = \pm \infty$, so the maximum will occur at $h'(\lambda)=0$ or else $\lambda=1$ if the derivative is never zero.  We have $h'(\lambda)=0$ at $\lambda = (\delta(\mathbf{c}, w) - \sqrt{6 n} - \log_2 w^*(\lambda))/n$.  Rearranging equation \ref{lb_all_one_side} we have
\begin{equation*} 
- \lambda^2 \frac{n}{2} + \lambda \left( \delta(\mathbf{c}, w) - \sqrt{6 n} - \log_2  w^*(\lambda) \right) + \log_2  w^*(\lambda)   \leq  \log_2 Z(w).
\end{equation*}
%and substituting the $\lambda = ((\delta(\mathbf{c}, w) - \sqrt{6 n}) - \log_2 w^*(\lambda))/n$ we have
%\[
%\frac{\left(\delta(\mathbf{c}, w) - \sqrt{6 n} - \log_2 w^*(\lambda)\right)^2}{2n} + \log_2 w^*(\lambda) \leq \log_2 Z(w).
%\]
Depending on the value of $\delta(\mathbf{c}, w)$ we have 3 separate regimes for the optimal lower bound on $Z(w)$.

 %Depending on the value of $\lambda$ and thus $w^*(\lambda)$, we have 3 separate regimes for the optimal lower bound, holding with probability greater than .95 for $\mathbf{c} \in \{-1,1\}^n$ sampled uniformly at random:

\begin{enumerate}
\item $\delta(\mathbf{c}, w) < n + \sqrt{6n} + \log_2 w_{min}$: In this case $h'(\lambda)=0$ at $\lambda = (\delta(\mathbf{c}, w) - \sqrt{6 n} - \log_2 w_{min})/n$ (note $\lambda < 1$ so that $w^*(\lambda) = w_{min}$) and the optimal lower bound is
\[
\frac{(\delta(\mathbf{c}, w) - \sqrt{6 n} - \log_2 w_{min})^2}{2n} + \log_2 w_{min} \leq \log_2 Z(w).
\]
We require that $\lambda >0$, but note that we can discard our bound and recompute with a new $\mathbf{c}$ if we find that $\lambda < 0$ for our computed value of $\delta(\mathbf{c}, w)$, as this can only happen with low probability when our slack bound is violated and we have estimated $\R(w)$ poorly with $\delta(\mathbf{c}, w)$.  Note that $\R(w) = \E_{\mathbf{c}}\left[ \max_{\mathbf{x}\in \{-1,1\}^n} \{\langle \mathbf{c},\mathbf{x} \rangle + \log_2 w(\mathbf{x})\} \right] \geq \E_{\mathbf{c}}\left[ \max_{\mathbf{x}\in \{-1,1\}^n} \{\langle \mathbf{c},\mathbf{x} \rangle + \log_2 w_{min}\} \right] = n + \log_2 w_{min}$, so $\R(w) - \log_2 w_{min} = n > 0$.

\item $n + \sqrt{6n} + \log_2 w_{min} < \delta(\mathbf{c}, w) < n + \sqrt{6n} + \log_2 w_{max}$: In this case $h'(\lambda)$ is never zero, so at $\lambda=1$ we have the optimal lower bound of% occurs When $(\delta(\mathbf{c}, w) - \sqrt{6 n} - \log_2 w_{min})/n > 1$ and $(\delta(\mathbf{c}, w) - \sqrt{6 n} - \log_2 w_{max})/n < 1$, we pick $\lambda = 1$ giving the optimal lower bound of 
\[
\delta(\mathbf{c}, w) - \sqrt{6 n} - \frac{n}{2} \leq \log_2 Z(w),
\]

\item $\delta(\mathbf{c}, w) > n + \sqrt{6n} + \log_2 w_{max}$: This case cannot occur because $\delta(\mathbf{c}, w) \leq n + \log_2 w_{max}$ by definition.

%\item $\lambda = (\delta(\mathbf{c}, w) - \sqrt{6 n} - \log_2 w_{max})/n > 1$ can never actually occur because we always have $\delta(\mathbf{c}, w) - \sqrt{6 n} < n + \log_2 w_{max}$, so the optimal lower bound is given by one of the two options above. 
\end{enumerate}

\end{proof}

We now illustrate how alternative choices of $\gamma$ could result in meaningful bounds.  From Equation~\ref{rademacher_upperBound} we have

\begin{align*}
\R(w^\gamma) \leq
\frac{1}{\lambda} \log_2 Z(w)  +\frac{\lambda\gamma - 1}{\lambda} \log_2  w^*(\lambda, \gamma) +\lambda \frac{ n}{2}\\
\E_{\mathbf{c}}\left[ \max_{\mathbf{x}\in \{-1,1\}^n} \left\{ \langle \mathbf{c},\mathbf{x} \rangle + \gamma \log_2w(\mathbf{x}) \right\} \right] \leq
\frac{1}{\lambda} \log_2 Z(w)  +\frac{\lambda\gamma - 1}{\lambda} \log_2  w^*(\lambda, \gamma) +\lambda \frac{ n}{2}\\
\E_{\mathbf{c}}\left[ \max_{\mathbf{x}\in \{-1,1\}^n} \left\{ \langle \mathbf{c},\mathbf{x} \rangle + \gamma \log_2\frac{w(\mathbf{x})}{w^*(\lambda, \gamma)} \right\} \right] \leq
\frac{1}{\lambda} \log_2 \frac{Z(w)}{w^*(\lambda, \gamma)} + \lambda \frac{ n}{2}.\\
\end{align*}
For sufficiently large $\gamma$ ($\gamma \geq \frac{1}{\lambda}$) we have $w^*(\lambda, \gamma) = w_{max}$, which makes $\log_2\frac{w(\mathbf{x})}{w^*(\lambda, \gamma)} \leq 0$ for all $x$.  Further, for sufficiently large $\gamma$, 
\[
 \argmax_{\mathbf{x}\in \{-1,1\}^n} \left\{ \langle \mathbf{c},\mathbf{x} \rangle + \gamma \log_2\frac{w(\mathbf{x})}{w^*(\lambda, \gamma)} \right\}  =  \argmax_{\mathbf{x}\in \{-1,1\}^n} \left\{ w(\mathbf{x}) \right\}
\]
(for all $\mathbf{c}$) and when a single element $\mathbf{x}\in \{-1,1\}^n$ has the unique largest weight ($w(\mathbf{x}) > w(\mathbf{y}) \forall x \neq y$)
\[
\E_{\mathbf{c}}\left[ \max_{\mathbf{x}\in \{-1,1\}^n} \left\{ \langle \mathbf{c},\mathbf{x} \rangle + \gamma \log_2\frac{w(\mathbf{x})}{w^*(\lambda, \gamma)} \right\} \right] = 0.
\]
Therefore, for sufficiently large $\gamma$ and
$
\lambda = \sqrt{\frac{2 \log_2{\frac{Z(w)}{w_{max}}}}{n}}
$, we have
\begin{align*}
\E_{\mathbf{c}}\left[ \max_{\mathbf{x}\in \{-1,1\}^n} \left\{ \langle \mathbf{c},\mathbf{x} \rangle + \gamma \log_2\frac{w(\mathbf{x})}{w^*(\lambda, \gamma)} \right\} \right] \leq
\frac{1}{\lambda} \log_2 \frac{Z(w)}{w^*(\lambda, \gamma)} + \lambda \frac{ n}{2}\\
0 \leq
\sqrt{2 n \log_2 \frac{Z(w)}{w_{max}}}\\
w_{max} \leq Z(w). \\
\end{align*}
This bound is trivial, however we picked $\gamma$ poorly.  To tighten this bound we would make $\gamma$ as small as possible subject to $\gamma \geq \frac{1}{\lambda}$ or alternatively we would make $\gamma$ as big as possible subject to $\gamma \leq \frac{1}{\lambda}$.  Note that we have used $
\lambda = \sqrt{\frac{2 \log_2{\frac{Z(w)}{w_{max}}}}{n}}
$, so gauranteeing $\gamma \geq \frac{1}{\lambda}$ or $\gamma \leq \frac{1}{\lambda}$ requires a bound on $Z(w)$.  We leave joint optimization over $\lambda$ and $\gamma$ for future work.
%\jonathan{It seems like we don't need to make $\gamma$ big, won't $\R(w^\gamma)$ always be bigger than zero.  double check}

\subsection{Upper Bound}
In this section we upper bound the sum $Z(w)$ by $\R(w)$ and apply our slack bound to obtain an upper bound on $Z(w)$ in terms of $\delta(\mathbf{c}, w)$.  Our proof technique is inspired by \cite{barvinok1997approximate}, who developed the method for bounding the sum $Z(w)$ of a weight function with values of either 0 or 1.  We generalize the proof to the weighted setting for any weight function $w:\{-1,1\}^n \rightarrow [0,\infty)$.  The principle underlying the method is that by dividing the space $\{-1,1\}^n$ in half $n$ times we arrive at a single weight.  By judiciously choosing which half of the space to recurse into at each step we can bound upper bound $Z(w)$.

Define the $j$ dimensional space $I_j = \{\mathbf{x}: \mathbf{x} \in \{-1,1\}^j\}$ for $j \geq 1$ and $I_0=\{(0)\}$.  For any vector $\mathbf{x} \in I_{j-1}$ with $j \geq 2$, define $\mathbf{x}^+,\mathbf{x}^{-} \in I_j$ as
\[
\mathbf{x}^+ = (x_1, x_2, \ldots, x_{j-1},1), \ \  \mathbf{x}^- = (x_1, x_2, \ldots, x_{j-1},-1).
\]
For the single element $(0) \in I_0$, define $(0)^+=(1)$ and $(0)^-=(-1)$. \jonathan{Is this notation clear with $\{-1,1\}^j$ referring to any vector of length $j$ and $(x_1, x_2, \ldots, x_{j-1},1)$ referring to a particular vector of length $j$?}

Given the weight function $w_j: I_j \rightarrow [0,\infty)$ (in this section we explicitly write $w_j$ to denote that the weight function has a $j$ dimensional domain while $w$ implicitly denotes $w_n$), define the weight functions
\[
%w^+_{j-1}:\{-1,1\}^{j-1} \rightarrow [0,\infty), x \mapsto w(x^+)
w^+_{j-1}:\{-1,1\}^{j-1} \rightarrow [0,\infty), \text{with } w^+_{j-1}(\mathbf{x}) = w_j(\mathbf{x}^+)
\]
and
\[
%w^{-}_{j-1}:\{-1,1\}^{j-1} \rightarrow [0,\infty), x \mapsto w(x^-).
w^-_{j-1}:\{-1,1\}^{j-1} \rightarrow [0,\infty), \text{with } w^-_{j-1}(\mathbf{x}) = w_j(\mathbf{x}^-)
\]
We have split the weights of our original weight function between two new weight functions, each with $j-1$ dimensional domains (two disjoint half spaces of our original $j$ dimensional domain). \jonathan{is this sloppy or clear?}  Now we relate the expectation $\R(w_j)$ to the expectations $\R(w_{j-1}^+)$ and $\R(w_{j-1}^-)$ in Lemmas \ref{recursive_Delta_1} and \ref{recursive_Delta_2}.

\begin{lemma} \label{recursive_Delta_1}
For $j \geq 1$ one has
\[
\R(w_{j-1}^+), \R(w_{j-1}^-) \leq \R(w_j)
\]
\end{lemma}
\begin{proof}
Given $\mathbf{x},\mathbf{c} \in I_{j-1}$ we have
\begin{align}
\langle \mathbf{c},\mathbf{x} \rangle + \log_2 w_{j-1}^+(\mathbf{x})= \frac{\left( \langle \mathbf{c}^+,\mathbf{x}^+ \rangle+ \log_2 w_j(\mathbf{x}^+) \right) + \left( \langle \mathbf{c}^-,\mathbf{x}^+ \rangle+ \log_2 w_j(\mathbf{x}^+) \right)}{2} \label{lhs_to_maximize} \\ 
\leq \frac{\max_{\mathbf{y} \in I_j} \left\{ \langle \mathbf{c}^+,\mathbf{y} \rangle+ \log_2 w_j(\mathbf{y}) \right\} + \max_{\mathbf{y} \in I_j} \left\{ \langle \mathbf{c}^-,\mathbf{y} \rangle + \log_2 w_j(\mathbf{y}) \right\}}{2} \nonumber \\
= \frac{\delta(\mathbf{c}^+,w_j)+\delta(\mathbf{c}^-,w_j)}{2}. \nonumber 
\end{align}
This inequality holds for any $\mathbf{x}$, so we can maximize the left hand side of Equation \ref{lhs_to_maximize} over $\mathbf{x}$ and get
\begin{equation*}
\delta(\mathbf{c},w_{j-1}^+) = \max_{\mathbf{x} \in I_{j-1}} \left\{ \langle \mathbf{c},\mathbf{x} \rangle + \log_2 w_{j-1}^+(\mathbf{x}) \right\} \leq \frac{\delta(\mathbf{c}^+,w_j)+\delta(\mathbf{c}^-,w_j)}{2}.
\end{equation*}
Now we average over $\mathbf{c} \in I_{j-1}$ and get
\begin{align*}
\R(w_{j-1}^+) &= \frac{1}{2^{j-1}} \sum_{\mathbf{c} \in I_{j-1}} \delta(\mathbf{c},w_{j-1}^+) 
\\
&\leq  
\frac{1}{2^{j-1}} \sum_{\mathbf{c} \in I_{j-1}} \frac{\delta(\mathbf{c}^+,w_j)+\delta(\mathbf{c}^-,w_j)}{2}\\
&= \frac{1}{2^{j}} \sum_{\mathbf{c} \in I_{j}} \delta(\mathbf{c},w_j)
=
\R(w_j)
\end{align*}
The proof for $\R(w_{j-1}^-) \leq \R(w_j)$ follows the same structure.
\end{proof}

\begin{lemma} \label{recursive_Delta_2}
For $j \geq 1$ one has
\[
\frac{\R(w_{j-1}^-)+\R(w_{j-1}^+)}{2} \leq \R(w_j) - 1
\]
\end{lemma}
\begin{proof}
Let $\mathbf{c},\mathbf{x}\in I_{j-1}$. Then
\begin{align*}
\label{lhs_to_maximize_2} \langle \mathbf{c},\mathbf{x} \rangle + \log_2 w_{j-1}^+(\mathbf{x})= \langle \mathbf{c}^+,\mathbf{x}^+ \rangle -1 +\log_2 w_{j-1}^+(\mathbf{x})\\
=\langle \mathbf{c}^+,\mathbf{x}^+ \rangle +  \log_2 w_j(\mathbf{x}^+) -1 \\
\leq \delta(\mathbf{c}^+,w_j) - 1
\end{align*}
The inequality $\langle \mathbf{c},\mathbf{x} \rangle + \log_2 w_{j-1}^+(\mathbf{x}) \leq \delta(\mathbf{c}^+,w_j) - 1$ holds for any $\mathbf{x}$.  Maximizing over $\mathbf{x}$ we get
\[
\delta(\mathbf{c},w_{j-1}^+) \leq \delta(\mathbf{c}^+,w_j) - 1.
\]
Similarly,
\begin{align*}
\langle \mathbf{c},\mathbf{x} \rangle + \log_2 w_{j-1}^-(\mathbf{x})= \langle \mathbf{c}^-,\mathbf{x}^- \rangle -1 + \log_2 w_{j-1}^-(\mathbf{x})\\
= \langle \mathbf{c}^-,\mathbf{x}^- \rangle + \log_2 w_j(\mathbf{x}^-) -1 \\
\leq \delta(\mathbf{c}^-,w_j) - 1,
\end{align*}
and maximizing over $\mathbf{x}$ we get
\[
\delta(\mathbf{c},w_{j-1}^-) \leq \delta(\mathbf{c}^-,w_j) - 1.
\]
Therefore
\begin{align*}
\frac{\delta(\mathbf{c},w_{j-1}^-)+\delta(\mathbf{c},w_{j-1}^+)}{2} \leq \frac{\delta(\mathbf{c}^-,w_j)+\delta(\mathbf{c}^+,w_j)}{2} - 1
\end{align*}
and averaging over $\mathbf{c} \in I_{j-1}$ we get
\begin{align*}
%\label{eq:geom}
\frac{\R(w_{j-1}^-)+\R(w_{j-1}^+)}{2} = \frac{1}{2^{j-1}} \sum_{\mathbf{c} \in I_{j-1}} \frac{\delta(\mathbf{c},w_{j-1}^-)+\delta(\mathbf{c},w_{j-1}^+)}{2}\\
\leq \frac{1}{2^{j-1}} \sum_{\mathbf{c} \in I_{j-1}} \left( \frac{\delta(\mathbf{c}^-,w_j)+\delta(\mathbf{c}^+,w_j)}{2} - 1 \right) \\ = \frac{1}{2^{j}} \left( \sum_{\mathbf{c} \in I_{j}}  \delta(\mathbf{c},w_j) \right) - 1 = \R(w_j) - 1.  \\
\end{align*}
\end{proof}

Equipped with our relations between the expectation $\R(w_j)$ and the expectations $\R(w_{j-1}^+)$ and $\R(w_{j-1}^-)$ from Lemmas \ref{recursive_Delta_1} and \ref{recursive_Delta_2}, we are prepared to upper bound $Z(w_n)$ by $\R(w_n)$.  To understand our strategy it is helpful to view the weight function $w_j:\{-1,1\}^{j} \rightarrow [0,\infty)$ as a binary tree with $2^j$ leaf nodes, each corresponding to a weight.  The weight functions $w_{j-1}^+$ and $w_{j-1}^-$ correspond to subtrees whose root nodes are the two children of the root node in the complete tree representing $w_j$.  Our strategy is to recursively divide the original weight function $w_n$ in half, picking one of the two subtrees based on their relative sizes (as measured by the sum of each subtree's weights and their weighted Rademacher complexities).  Eventually we arrive at a leaf, corresponding to a single weight.  This allows us to relate the weighted Rademacher complexity of the original weight function (or the entire tree) to $Z(w_n)$ and the weight of this single leaf.  The problem of choosing which subtree to pick at each step based on their relative sizes is computationally intractable.  To avoid this difficulty we do not explicitly find the leaf, but instead conservatively pick either the largest or smallest weight in the entire tree, which gives us an upper bound on $Z(w_n)$ in terms of $\R(w_n)$.  After upper bounding $Z(w_n)$ by $\R(w_n)$ we apply our slack bound to obtain an upper bound on $Z(w_n)$ in terms of $\delta(\mathbf{c}, w_n)$. 
%divide the space $I_n$ in half.  We pick which half space to recurse into at each step based on the sum of each half space's weights and the weighted Rademacher complexity of each half space.  

\begin{lemma}  Assuming $Z(w) > 0$, for $\mathbf{c} \in \{-1,1\}^n$ sampled uniformly at random, the following bound holds with probability greater than .95:
\begin{align*}
\log_2 Z(w) \leq
\begin{cases}
	 \log_2 \left(\frac{1 - \beta_{opt}}{\beta_{opt}}\right) \left( \delta(\mathbf{c}, w) + \sqrt{6 n} - \log_2 w_{min} \right) - n \log_2 \left( 1 - \beta_{opt} \right) + \log_2 w_{min}, & \text{if } 0 < \beta_{opt} < \frac{1}{3} \\  
	\log_2 \left(\frac{1 - \beta_{opt}}{\beta_{opt}}\right) \left( \delta(\mathbf{c}, w) + \sqrt{6 n} - \log_2 w_{max} \right) - n \log_2 \left( 1 - \beta_{opt} \right) + \log_2 w_{max}, & \text{if } \frac{1}{3} < \beta_{opt} < \frac{1}{2} \\  
	 \delta(\mathbf{c}, w) + \sqrt{6 n} + n \log_2 \left( \frac{3}{2} \right) \approx \delta(\mathbf{c}, w) + \sqrt{6 n} + .58 n,  & \text{if } \beta_{opt} = \frac{1}{3}\\
     \log_2 w_{max} + n,  & \text{if } \beta_{opt} = \frac{1}{2}\\
\end{cases}
\end{align*}

where 

\begin{align*}
\beta_{opt} = 
\begin{cases}
	 \frac{\delta(\mathbf{c}, w) + \sqrt{6 n} - \log_2 w_{min}}{n}, & \text{if } 0 < \frac{\delta(\mathbf{c}, w) + \sqrt{6 n} - \log_2 w_{min}}{n} < \frac{1}{3} \\  
	\frac{\delta(\mathbf{c}, w) + \sqrt{6 n} - \log_2 w_{max}}{n}, & \text{if } \frac{1}{3} < \frac{\delta(\mathbf{c}, w) + \sqrt{6 n} - \log_2 w_{max}}{n} < \frac{1}{2} \\   
	\frac{1}{2}, & \text{if } \frac{1}{2} < \frac{\delta(\mathbf{c}, w) + \sqrt{6 n} - \log_2 w_{max}}{n} \\    
	\frac{1}{3}, & \text{otherwise } %\frac{\delta(\mathbf{c}, w) + \sqrt{6 n} - \log_2 w_{max}}{n} \leq \frac{1}{3} \\    
\end{cases}
\end{align*}

\end{lemma}
\begin{proof}

Let $\beta$ be a parameter, to be set later, such that $0 < \beta \leq 1/2$.  \jonathan{We require that $\beta > 0$ so $w(\bar{x}) > 0$.}  We construct a sequence of weight functions $w_n, w_{n-1}, \ldots, w_1,w_0$ where $w_j:\{-1,1\}^{j} \rightarrow [0,\infty)$.  Starting with our original weight function $w_n$, we use two rules to decide whether $w_{j-1} = w_{j-1}^-$ or $w_{j-1}=w_{j-1}^+$.

\textbf{Rule 1}: Given $w_j$, if $\min \{\sum_\mathbf{y} w_{j-1}^+(\mathbf{y}), \sum_\mathbf{y} w_{j-1}^-(\mathbf{y})\} < \beta \sum_\mathbf{y} w_j(\mathbf{y})$, we let 
\[
w_{j-1} = \argmax_{w_{j-1}^+,w_{j-1}^-} \left\{ \sum_\mathbf{y} w_{j-1}^+(\mathbf{y}), \sum_\mathbf{y} w_{j-1}^-(\mathbf{y}) \right\}
\]
%be the one with larger partition function.

\textbf{Rule 2}: Given $w_j$, if $\min \{\sum_\mathbf{y} w_{j-1}^+(\mathbf{y}), \sum_\mathbf{y} w_{j-1}^-(\mathbf{y})\} \geq \beta \sum_\mathbf{y} w_j(\mathbf{y})$, we let 
\[
w_{j-1} = \argmin_{w_{j-1}^+,w_{j-1}^-} \left\{ \R(w_{j-1}^+), \R(w_{j-1}^-) \right\}
\]

Note that $\R(w_{0}) = \log_2 w_n(\overline{\mathbf{x}})$ for some $\overline{\mathbf{x}} \in I_n$.  That is, after dividing our original space with $2^n$ states in half $n$ times, we are left with a single state.  Now, given that $Z(w)>0$, rule 1 guarantees $w_n(\overline{\mathbf{x}})>0$.  As proof by contradiction, assume that $w_n(\overline{\mathbf{x}})=0$.  This requires that for some for some integer $i$ (with $0 < i \leq n$) we have $0 < \sum_\mathbf{y} w_{i}(\mathbf{y})$ and $0 = \sum_\mathbf{y} w_{i-1}(\mathbf{y})$, but following rule 1 makes this impossible.

Our first step is to relate the weighted Rademacher complexity $\R(w_n)$ to the final leaf weight $w_n(\overline{\mathbf{x}})$ based on the number of times we use rule 2 when dividing the original tree.  Every time we use rule 2
\[
\R(w_{j-1}) = \min_{w_{j-1}^+,w_{j-1}^-} \{\R(w_{j-1}^+), \R(w_{j-1}^-)\}.
\]
By Lemma~\ref{recursive_Delta_2}, $\R(w_{j})$
%\[
%\R(w_{j}) \geq \frac{\R(w_{j-1}^-)+\R(w_{j-1}^+)}{2} + 1
%\]
is at least as large as the average of $\R(w_{j-1}^-)$ and $\R(w_{j-1}^+)$ plus one, making it at least as large as the minimum of $\R(w_{j-1}^-)$ and $\R(w_{j-1}^+)$ plus one.  Therefore $\R(w_{j}) \geq \R(w_{j-1}) +1$ whenever we use rule 2.  By lemma \ref{recursive_Delta_1} we have $\R(w_{j}) \geq \R(w_{j-1})$ regardless of whether we use rule 1 or 2.  Let $m$ be the number of times we have used Rule 2, then
\begin{equation}\label{bound_rademacher_m}
\R(w)  = \R(w_{n})  \geq \R(w_{0}) + m =  \log_2 w_n(\overline{\mathbf{x}}) + m.
\end{equation}

Now we relate the number of times we use rule 2 to the sum $Z(w_n)$.  Observe that $\sum_{\mathbf{y} \in I_{j-1}} w_{j-1}^+(\mathbf{y}) + \sum_{\mathbf{y} \in I_{j-1}} w_{j-1}^-(\mathbf{y}) = \sum_{\mathbf{z} \in I_j} w_{j} (\mathbf{z})$; therefore if we have used rule 1 we must have $\sum_\mathbf{y} w_{j-1}(\mathbf{y}) \geq (1-\beta) \sum_\mathbf{z} w_{j}(\mathbf{z})$ because we picked the largest of the two subtrees. If we have used rule 2 then $\sum_\mathbf{y} w_{j-1}(\mathbf{y}) \geq \beta \sum_\mathbf{z} w_{j}(\mathbf{z})$, because by definition we use rule 2 when the two subtrees both carry at least the fraction $\beta$ of the total weight. Therefore
% \[
% w_{max} \geq w_n(\overline{\mathbf{x}})=\sum_{\mathbf{y} \in I_0} w_0(\mathbf{y}) \geq \left(1-\beta\right)^{n-m} \left(\beta\right)^m Z(w) = \left(1-\beta\right)^{n-m} \left(\beta\right)^m 2^{\alpha n}w_{max}
% \]
\[
w_n(\overline{\mathbf{x}})=\sum_{\mathbf{y} \in I_0} w_0(\mathbf{y}) \geq \left(1-\beta\right)^{n-m} \beta^m \sum_{\mathbf{y} \in I_n} w_n(\mathbf{y}) =
\left(1-\beta\right)^{n} \left(\frac{\beta}{1-\beta}\right)^m Z(w_n).
\]
Taking logarithms \jonathan{note, this logarithm does not have to be base two, while the logarithm in $\R(w)  = \R(w_{n})  \geq  \log_2 w_n(\overline{\mathbf{x}}) + m$ \textbf{must} be base 2 (unless we rescale).  Changing the logarithm base here can't improve the bound, so we pick base 2 to match the other log.  Could discuss more/handle differently to avoid confusion about which log bases are important because it's somewhat confusing. }
\[
\log_2 w_n(\overline{\mathbf{x}}) \geq 
n \log_2 \left(1-\beta\right) + m \log_2 \left(\frac{\beta}{1-\beta}\right) + \log_2  Z(w_n) 
\]
\[
- m \log_2 \left(\frac{\beta}{1-\beta}\right)  \geq 
n \log_2 \left(1-\beta\right)  + \log_2  Z(w_n) -\log_2 w_n(\overline{\mathbf{x}})
\]
\begin{equation} \label{get_trivial_UB}
m \log_2 \left(\frac{1 - \beta}{\beta}\right)  \geq 
n \log_2 \left(1-\beta\right)  + \log_2  Z(w_n) -\log_2 w_n(\overline{\mathbf{x}}).
\end{equation} 
Note that $\log_2 \left(\frac{1 - \beta}{\beta}\right) >0$ for $0<\beta<1/2$ so
\begin{equation} \label{bound_m_Z}
m  \geq 
\frac{ n \log_2 \left(1-\beta\right)  + \log_2  Z(w_n) -\log_2 w_n(\overline{\mathbf{x}})}{\log_2 \left(\frac{1 - \beta}{\beta}\right) }
\end{equation}
and by combining Equations~\ref{bound_rademacher_m} and~\ref{bound_m_Z} we have
\begin{equation*}%\label{rademacher_lower_bound}
\R(w) \geq  \log_2 w_n(\overline{\mathbf{x}}) +  \frac{ n \log_2 \left(1-\beta\right)  + \log_2  Z(w_n) -\log_2 w_n(\overline{\mathbf{x}})}{\log_2 \left(\frac{1 - \beta}{\beta}\right) }.
\end{equation*}
Applying the slack bound from Equation~\ref{bound_max_by_expectation} we have%
\footnote{This bound is scale invariant; if we scale the weight function by a constant $a$ so that $w'(\mathbf{x}) = a w(\mathbf{x})$ then $\R(w') = \log_2 a + \R(w)$, $\log_2 Z(w_n') = \log_2 a + \log_2 Z(w_n)$, and $\log_2 w'_n(\overline{\mathbf{x}}) = \log_2 a + \log_2 w_n(\overline{\mathbf{x}})$ so the bound remains unchanged.}
\begin{equation}\label{delta_lower_bound}
\delta(\mathbf{c}, w) \geq  \log_2 w_n(\overline{\mathbf{x}}) +  \frac{ n \log_2 \left(1-\beta\right)  + \log_2  Z(w_n) -\log_2 w_n(\overline{\mathbf{x}})}{\log_2 \left(\frac{1 - \beta}{\beta}\right) } - \sqrt{6n}
\end{equation}
with probability greater than .95.

Now we can choose $\beta$ to optimize this bound.  %Recall, we choose to use rule 1 or rule 2 based on whether $\min \{\sum_\mathbf{y} w_{j-1}^+(\mathbf{y}), \sum_\mathbf{y} w_{j-1}^-(\mathbf{y})\}$ is smaller or larger than $\beta \sum_\mathbf{y} w_j(\mathbf{y})$.  
Note that the bound in Equation~\ref{delta_lower_bound} contains the term $w_n(\overline{\mathbf{x}})$, however it is computationally intractable to explicitly find this leaf following the procedure outlined above.  Instead we conservatively use either the smallest or largest weight depending on the value of $\beta$.  If these weights cannot be computed we may alternatively pick $\beta = 1/3$ to eliminate $w_n(\overline{\mathbf{x}})$ from the bound.  Depending on the value of $\beta$ we have 4 cases outlined below:
\begin{enumerate}
\item When $\beta = 1 / 3$, the quantity $\log_2 \left(\frac{1 - \beta}{\beta}\right) = 1$ and
\[
\delta(\mathbf{c}, w) \geq  \log_2  Z(w_n) - n \log_2 (3/2) - \sqrt{6n}.
\]

\item When $0 < \beta < \frac{1}{3}$, the quantity $1 - \frac{1}{\log_2 \left(\frac{1 - \beta}{\beta}\right)} > 0$ and

\[
\delta(\mathbf{c}, w) \geq    \frac{ n \log_2 \left(1-\beta\right)  + \log_2  Z(w_n)}{\log_2 \left(\frac{1 - \beta}{\beta}\right) } + \log_2 w_{min} \Bigg(1 - \frac{1}{\log_2 \left(\frac{1 - \beta}{\beta}\right)} \Bigg) - \sqrt{6n}
\]

\item When $\frac{1}{3} < \beta < .5$, the quantity $1 - \frac{1}{\log_2 \left(\frac{1 - \beta}{\beta}\right)} < 0$ and

\[
\delta(\mathbf{c}, w) \geq    \frac{ n \log_2 \left(1-\beta\right)  + \log_2  Z(w_n)}{\log_2 \left(\frac{1 - \beta}{\beta}\right) } + \log_2 w_{max} \Bigg(1 - \frac{1}{\log_2 \left(\frac{1 - \beta}{\beta}\right)} \Bigg) - \sqrt{6n}.
\]

\item When $\beta = .5$, the quantity $\frac{1 - \beta}{\beta} = 0$ and we recover the trivial bound $\log_2 Z(w_n) \leq \log_2 w_{max} + n$ from equation \ref{get_trivial_UB}.

\end{enumerate}
To choose the best value for $\beta$ we minimize the upper bound on $\log_2 Z(w_n)$ with respect to $\beta$.  The upper bound on $\log_2 Z(w_n)$ is
\begin{align*}
&& \log_2 Z(w_n) \leq L(\beta) =  \log_2 \left(\frac{1 - \beta}{\beta}\right) \left( \delta(\mathbf{c}, w) + \sqrt{6 n} - \log_2 w^*(\beta) \right) - n \log_2 \left( 1 - \beta \right) + \log_2 w^*(\beta) \\
&& =  a \log_2 \left(\frac{1 - \beta}{\beta}\right) - n \log_2 \left( 1 - \beta \right) + \log_2 w^*(\beta),
\end{align*}
with
\[
a = \delta(\mathbf{c}, w) + \sqrt{6 n} - \log_2 w^*(\beta),
\]
where $w^*(\beta) = w_{min} = \min_\mathbf{x} \{w(\mathbf{x}) : w(\mathbf{x}) > 0\}$ for $\beta < \frac{1}{3}$ and $w^*(\beta) = w_{max} = \max_\mathbf{x} w(\mathbf{x})$ for $\beta > \frac{1}{3}$.  Differentiating we get
\[
L'(\beta) =  \frac{a \left(-\frac{1-\beta}{\beta^2}-\frac{1}{\beta}\right) \beta}{1-\beta}+\frac{n}{1-\beta}
\]
\[
L''(\beta) = \frac{a \beta \left(-\frac{1-\beta}{\beta^2}-\frac{1}{\beta}\right)}{(1-\beta)^2}+\frac{a
   \left(-\frac{1-\beta}{\beta^2}-\frac{1}{\beta}\right)}{1-\beta}+\frac{a \left(\frac{2
   (1-\beta)}{\beta^3}+\frac{2}{\beta^2}\right) \beta}{1-\beta}+\frac{n}{(1-\beta)^2}
\]

The first derivative has a root at $\beta = \frac{a}{n}$, with $L''(\frac{a}{n}) = \frac{n^3}{a (n-a)}$.  By definition the only meaningful values of $\beta$ are $0 < \beta \leq .5$, so either $0 < a < n/2$ and the second derivative is positive making this root a minimum or $\frac{a}{n}$ is outside our valid range for $\beta$ and the minimum occurs at an endpoint of the range.  Note that for $a>0$ we have $\lim_{\beta \rightarrow 0} L(\beta) = \infty$, so the minimum never occurs at the endpoint $\beta = 0$ when $a>0$.  If $a \leq 0$ then our slack bound has been violated and  we have estimated $\R(w)$ poorly with $\delta(\mathbf{c}, w)$, so it is appropriate to sample a new $\mathbf{c}$ and recompute $\delta(\mathbf{c}, w)$.  (To see this, note that $\R(w) = \E_{\mathbf{c}}\left[ \max_{\mathbf{x}\in \{-1,1\}^n} \{\langle \mathbf{c},\mathbf{x} \rangle + \log_2 w(\mathbf{x})\} \right] \geq \E_{\mathbf{c}}\left[ \max_{\mathbf{x}\in \{-1,1\}^n} \{\langle \mathbf{c},\mathbf{x} \rangle + \log_2 w_{min}\} \right] = n + \log_2 w_{min}$, so $\R(w) - \log_2 w_{min} = n > 0$.)  Also, $\lim_{\epsilon \rightarrow 0} L(\frac{1}{3} + \epsilon) = \lim_{\epsilon \rightarrow 0} L(\frac{1}{3} - \epsilon)$ (at $\beta = 1/3$ the term $w^*(\beta)$ doesn't appear in the bound so it doesn't matter whether $w^*(\beta) = w_{min}$ or $w^*(\beta) = w_{max}$).  Assuming we have sampled $\mathbf{c}$ such that $a > 0$, this means the optimal value of $\beta$ that minimizes our upper bound is:
\begin{align*}
\beta_{opt} = 
\begin{cases}
	 \frac{\delta(\mathbf{c}, w) + \sqrt{6 n} - \log_2 w_{min}}{n}, & \text{if } 0 < \frac{\delta(\mathbf{c}, w) + \sqrt{6 n} - \log_2 w_{min}}{n} < \frac{1}{3} \\  
	\frac{\delta(\mathbf{c}, w) + \sqrt{6 n} - \log_2 w_{max}}{n}, & \text{if } \frac{1}{3} < \frac{\delta(\mathbf{c}, w) + \sqrt{6 n} - \log_2 w_{max}}{n} < \frac{1}{2} \\   
	\frac{1}{2}, & \text{if } \frac{1}{2} < \frac{\delta(\mathbf{c}, w) + \sqrt{6 n} - \log_2 w_{max}}{n} \\    
	\frac{1}{3}, & \text{otherwise } %\frac{\delta(\mathbf{c}, w) + \sqrt{6 n} - \log_2 w_{max}}{n} \leq \frac{1}{3} \\    
\end{cases}
\end{align*}
%Note that if $a>n$ and $L''(a/n) < 0$, $L'(\beta) < 0$ for $1/3 < \beta < 1$ and the minimum of $L(\beta)$ will still occur at $\beta = 1/2$.  
Our upper bound on $Z(w_n)$ is given by 

\begin{align*}
\log_2 Z(w_n) \leq
\begin{cases}
	 \log_2 \left(\frac{1 - \beta_{opt}}{\beta_{opt}}\right) \left( \delta(\mathbf{c}, w) + \sqrt{6 n} - \log_2 w_{min} \right) - n \log_2 \left( 1 - \beta_{opt} \right) + \log_2 w_{min}, & \text{if } 0 < \beta_{opt} < \frac{1}{3} \\  
	\log_2 \left(\frac{1 - \beta_{opt}}{\beta_{opt}}\right) \left( \delta(\mathbf{c}, w) + \sqrt{6 n} - \log_2 w_{max} \right) - n \log_2 \left( 1 - \beta_{opt} \right) + \log_2 w_{max}, & \text{if } \frac{1}{3} < \beta_{opt} < \frac{1}{2} \\  
	 \delta(\mathbf{c}, w) + \sqrt{6 n} + n \log_2 \left( \frac{3}{2} \right) \approx \delta(\mathbf{c}, w) + \sqrt{6 n} + .58 n,  & \text{if } \beta_{opt} = \frac{1}{3}\\
     \log_2 w_{max} + n,  & \text{if } \beta_{opt} = \frac{1}{2}\\
\end{cases}
\end{align*}

\end{proof}

% \begin{itemize}
% \item the last approximation is not rigorous, but it shows that if the distribution is uniform, it should still perform well.
% \item bound can be improved
% \end{itemize}

\subsubsection{Tightening the Slack Bound}
We can improve our high probability bounds on $\delta(\mathbf{c}, w)$ in terms of $Z(w)$ by generating $k$ independent vectors $\mathbf{c}_1, \mathbf{c}_2, \dots, \mathbf{c}_{k} \in \{-1, 1\}^n$, applying the optimization oracle from Assumption~\ref{opt_oracle} to each, and taking the mean $\bar{\delta}_k(w) = (\delta(\mathbf{c}_1, w) + \dots + \delta(\mathbf{c}_{k}, w))/k$.  This gives the bounds
\[
  \log_2 w^*(\beta) +  \frac{ n \log_2 \left(1-\beta\right)  + \log_2  Z(w) -\log_2 w^*(\beta)}{ \log_2 \left(\frac{1 - \beta}{\beta}\right) } - \sqrt{\frac{6 n}{k}} \leq \bar{\delta}_k(w) \leq \frac{1}{\lambda} \log_2 Z(w)  +\frac{\lambda \gamma - 1}{\lambda} \log_2  w^*(\lambda, \gamma) +\lambda \frac{ n}{2} + \sqrt{\frac{6 n}{k}}.
\]
In the last two sections we inverted these bounds and optimized over $\beta$ and $\lambda$ to obtain high probability bounds on $Z(w)$ in terms of $\delta(\mathbf{c}, w)$.  This process is unchanged, with the exceptions that we replace $\delta(w, \mathbf{c})$ (computed from a single $\mathbf{c} \in \{-1,1\}^n$) with $\bar{\delta}_k(w)$ and the term $\sqrt{6n}$ with $\sqrt{\frac{6 n}{k}}$.

%Note that we now bound the partition function by choosing optimal values of $\beta$ and $\lambda$ with respect to $\bar{\delta}_k(w)$ rather than $\delta(w, \mathbf{c})$ computed from a single $\mathbf{c} \in \{-1,1\}^n$.

Proof: recall the weight function $w:\{-1,1\}^n \rightarrow [0,\infty)$.  Let's define a new weight function $w':\{-1,1\}^{n k} \rightarrow [0,\infty)$ as $w'(\mathbf{c}_1, ..., \mathbf{c}_{k}) = w(\mathbf{c}_1)\times \dots \times w(\mathbf{c}_{k})$ for $\mathbf{c}_1, \dots, \mathbf{c}_{k} \in \{-1, 1\}^n$.  The sum of this new function's weights is%The partition function of $w'$ is
\begin{align*}
Z(w') = \sum_{\mathbf{x} \in \{-1,1\}^{k n}} w'(\mathbf{x}) \\
= \sum_{\mathbf{x}_1 \in \{-1,1\}^{ n}} \dots \sum_{\mathbf{x}_{k} \in \{-1,1\}^{n}} w(\mathbf{x}_1)\times \dots \times w(\mathbf{x}_{k}) \\
= \sum_{\mathbf{x}_1 \in \{-1,1\}^{ n}} w(\mathbf{x}_1)\times \dots \times \sum_{\mathbf{x}_{k} \in \{-1,1\}^{n}} w(\mathbf{x}_{k}) \\
= Z(w)^{k}.
\end{align*}
Also note that the largest and smallest non-zero weights in $w'$ are $w'_{max}=w^{k}_{max}$ and $w'_{min}=w^{k}_{min}$.   Define $\mathbf{c}' \in \{-1,1\}^{nk}$ as $\mathbf{c}' = (\mathbf{c}_1, ..,\mathbf{c}_{k})$.  The value $\delta(w', \mathbf{c}')$ for our new weight function is now
\begin{align*}
\delta(w', \mathbf{c}') = \max_{\mathbf{x}' \in \{-1,1\}^{n k}} \left\{ \log_2 w'(\mathbf{x}') + \langle \mathbf{c}',\mathbf{x}' \rangle \right\} \\
=  \max_{\mathbf{x}_1 \in \{-1,1\}^{n}} \left\{ \log_2 w(\mathbf{x}_1) + \langle \mathbf{c}_1,\mathbf{x}_1 \rangle \right\} + \dots + \max_{\mathbf{x}_{k} \in \{-1,1\}^{n}} \left\{ \log_2 w(\mathbf{x}_{k}) + \langle \mathbf{c}_{k},\mathbf{x}_{k} \rangle \right\} \\
= k \bar{\delta}_k(w).
\end{align*}
We lower bound $\bar{\delta}_k(w)$ by applying the bound from Equation~\ref{delta_lower_bound} to $w'$ (recalling that either $w^*(\beta) = w_{min}$ or $w^*(\beta) = w_{max}$) and find that

\begin{align*}
\bar{\delta}_k(w) = \frac{\delta(w', \mathbf{c}')}{k} \geq  \frac{\log_2 w^*(\beta)^{k}}{k} +  \frac{ n k \log_2 \left(1-\beta\right)  + \log_2  Z(w)^{k} -\log_2 w^*(\beta)^{k}}{k \log_2 \left(\frac{1 - \beta}{\beta}\right) } - \frac{\sqrt{6 nk}}{k} \\
=  \log_2 w^*(\beta) +  \frac{ n \log_2 \left(1-\beta\right)  + \log_2  Z(w) -\log_2 w^*(\beta)}{\log_2 \left(\frac{1 - \beta}{\beta}\right) } - \sqrt{\frac{6 n}{k}}.
\end{align*}
Similarly, for the upper bound on $\bar{\delta}_k(w)$ we apply the bound from Equation~\ref{delta_UB} to $w'$ and find that
\begin{align*}
\bar{\delta}_k(w) = \frac{\delta(w', \mathbf{c}')}{k} \leq \frac{1}{k} \left(\frac{1}{\lambda} \log_2 Z(w)^{k}  +\frac{\lambda \gamma - 1}{\lambda} \log_2  {w^*(\lambda, \gamma)}^{k} +\lambda \frac{n k}{2} + \sqrt{6 n k} \right)\\
= \frac{1}{\lambda} \log_2 Z(w)  +\frac{\lambda \gamma - 1}{\lambda} \log_2  w^*(\lambda, \gamma) +\lambda \frac{ n}{2} + \sqrt{\frac{6 n}{k}}.\\
\end{align*}

\end{document}